\documentclass{amsart}
\usepackage{amsfonts}
\usepackage{latexsym}
\usepackage[all]{xy}
\usepackage{graphicx}
\usepackage[T1]{fontenc}
\usepackage{pifont}
\usepackage{verbatim}
\usepackage{amssymb}
\usepackage{amscd}
\usepackage{graphics}
\usepackage{graphpap}
\usepackage{float}

\newtheorem{definicion}{Definition}[section]
\newtheorem{nota}[definicion]{Remark}
\newtheorem{prop}[definicion]{Proposition}

\newtheorem{obs}[definicion]{Remark}
\newtheorem{teorema}[definicion]{Theorem}
\newtheorem{cor}[definicion]{Corollary}
\newtheorem{ejp}[definicion]{Example}

\newcommand{\co}{\ensuremath{\colon}} 
\newcommand{\bn}{\ensuremath{\mathbb{N}}} 
\newcommand{\br}{\ensuremath{\mathbb{R}}} 

\begin{document}

\title{A density-sensitive hierarchical clustering method.}



\author{\'{A}lvaro Mart\'{i}nez-P\'{e}rez}       


\address{Departamento de An\'{a}lisis Econ\'{o}mico y Finanzas. Universidad de Castilla-La Mancha.\\ 
							Avda. Real F\'{a}brica de Seda, s/n. 45600 Talavera de la Reina 
Toledo.  Spain \\
              Tel.:  +34 902 204 100\\
              Fax: +34  902 204 1300}
              \email{alvaro.martinezperez@uclm.es}           

\date{}

\begin{abstract}
We define a hierarchical clustering method: $\alpha$-unchaining single linkage or $SL(\alpha)$. The input of this algorithm is a finite metric space and a certain parameter $\alpha$. This method is sensitive to the density of the distribution and offers some solution to the so called chaining effect. We also define a modified version, $SL^*(\alpha)$, to treat the chaining through points or small blocks. We study the theoretical properties of these methods and offer some theoretical background for the treatment of chaining effects.
\end{abstract}

\maketitle

\begin{footnotesize}
Keywords: Hierarchical clustering, single linkage, chaining effect, weakly unchaining, $\alpha$-bridge-unchaining. 
\end{footnotesize}

\begin{footnotesize} MSC: 62H30, 68T10.
\end{footnotesize}

\tableofcontents

\section{Introduction}

A \textit{clustering method} is an algorithm that takes as input a finite space with a distance function (typically, a finite metric space) $(X,d)$ and gives as output a partition of $X$. 

Kleinberg discussed in \cite{Kle} the problem of clustering in an axiomatic way and proposed a few basic properties that a clustering scheme should hold. Then, he proved that no standard clustering scheme satisfying this conditions simultaneously can exist. This does not imply the impossibility of defining a consistent standard clustering algorithm. Kleinberg's impossibility only holds when the unique input in the algorithm is the space and the set of distances. It can be avoided including, for example, the number of clusters to be obtained as part of the input. See \cite{ABL_10} and \cite{ZB}. Also, Ackerman and Ben-David, see \cite{AB_08},  showed that these axioms are consistent in the setting of clustering quality measures.

Carlsson and M\'{e}moli, see \cite{CM}, studied the analogous problem for clustering schemes that yield hierarchical decompositions instead of a certain partition of the space. See also \cite{CM3} and \cite{CM4}. Hierarchical clustering methods also take as input a finite metric space but the output is a hierarchical family of partitions of $X$.

They approach the subject focusing on a theoretical basis for the study of hierarchical clustering ($HC$). In the spirit of Kleinberg's result, they define a few reasonable conditions that a $HC$ method should hold. 
They prove that the unique $HC$ method satisfying three basic conditions is (the well-known) single linkage hierarchical clustering, $SL$ $HC$. Ackerman and Ben-David, see \cite{AB_11}, proved also a characterization of the class of linkage-based algorithms, including $SL$. See also \cite{ABL_12}. In the setting of partitional (standard) clustering see \cite{ZB}.  

Carlsson and M\'{e}moli also study the theoretical properties of $SL$ $HC$ obtaining some interesting results. The main advantage seems to be that this method enjoys some sort of stability which is defined by means of the Gromov-Hausdorff distance. However, the main weakness of $SL$ $HC$ is the so called \textit{chaining effect} which may merge clusters that, in practice, \textit{should be detected} by the algorithm and kept separated. One way to address this difficulty is to take into account the density. In a preprint entitled \textit{Multiparameter Hierarchical Clustering Methods} the same authors do this by including in the input of the algorithm a function that provides that information.

Our aim is to define a $HC$ method which offers some solution to this particular weakness without including any extra information. The first challenge is that the concept of \textit{clusters that should be detected} by the algorithm depends on the characteristics of the problem under study. The same happens with what we may consider the \textit{undesired} chaining effect. The definition from Lance and Williams, \cite{LW}, makes reference to the higher tendency of the points to add to a pre-existing group rather than defining the nucleus of a new group or joining to another single point. Our algorithm is oriented to another aspect of the chaining effect which is the tendency to merge two clusters when the minimal distance between them is small (even though they may have dense cores which are clearly distant apart). This is typically the problem of $SL$ $HC$. Also, we include as an undesired chaining effect the case of two big clusters joined by a chain of points or small clusters. These isolated points or small clusters may be interpreted as noise in the sample and we might want to distinguish the big picture and ignore their effect. This idea is closer to the type of chaining effect considered by Wishart in \cite{W}. 

There exist other linkage-based methods that enjoy some sort of sensitivity to density and offer some resistance to these chaining effects as average linkage, $AL$, or complete linkage, $CL$. These methods are extensively used in practice. However, although the main problem of the chaining effect of $SL$ $HC$ is reduced, there appears another effect that might be unwanted too. In these methods the distance between a point and a cluster is greater than the minimal distance. Therefore, they have a tendency to merge isolated points before joining them to pre-existing big clusters. Also, these methods are proved to be extremely unstable in the sense that small perturbations on the data yield very different dendrograms.  

Herein, we define a new $HC$ method on the basis of $SL$: \textit{$\alpha$-unchaining single linkage} or $SL(\alpha)$. The definition of $SL(\alpha)$ is based in the dimension of the Rips complexes defined by the points of $X$. These complexes contain some information about the density distribution of the sample. This allows us to define a density sensitive algorithm such that the input is just the set of distances between the points and a fixed parameter $\alpha\in \bn$. The parameter determines how sensitive the method will be to the chaining effect. 

To treat the chaining through single points or smaller blocks, we define another version of the method, $SL^*(\alpha)$, by adding an extra condition on $SL(\alpha)$. 

It is worth mentioning that  Ester, Kriegel, Sander and Xu, see \cite{EKSX}, defined a standard clustering method called DBSCAN. DBSCAN is also a density-sensitive algorithm where the input is just the data set and some parameters. Although this is not hierarchical clustering algorithm we include a short discussion comparing the type of clusterings detected by this method with the clusterings that $SL(\alpha)$ and $SL^*(\alpha)$ may detect in the levels of the dendrogram. 

This paper  intends to give a theoretical basis to the study of the problem. So, instead of checking the algorithm on examples of real data we rather try to find general properties characterizing what would be an undesired chaining of two blocks and how good is the algorithm detecting and unchaining them. However, we include several examples where the unchaining properties can be explicitly seen in the resulting dendrogram.

To the study the chaining effect we define the concepts of \textit{chained subsets}  and \textit{subsets chained by smaller blocks} as situations of minimal chaining so that they contain what we consider the problematic examples. Nevertheless, there may be many examples of chained subsets which should be clearly merged and there is margin to be more restrictive. In such context, a $HC$ method is \textit{strongly chaining} if every pair of chained subsets are always merged before they appear contained in different clusters. A $HC$ method is \textit{completely chaining} if, in addition, every pair of subsets chained through smaller blocks are merged before they appear contained in different clusters. Thus, strongly chaining methods and completely  chaining methods are extremely sensitive to these effects. This is the case, for example, of $SL$ $HC$. See Theorem \ref{Th: complete chaining}.

We define also precise conditions to define what we consider two blocks that should necessarily appear as independent blocks at some point. The definition considers two blocks that have dense cores and such that the minimal distance between them is small only because of a single pair of points. In particular, this pair of points creates a chaining between their dense cores. See Figure \ref{Ejp Chained edge}. This is a particular, more restrictive, example of chained subsets. We say that a $HC$ method is \textit{weakly unchaining} if, at least, it distinguishes that pair of blocks. Then, we prove that $SL(\alpha)$ satisfies this condition while other methods which are not strongly chaining as $AL$ and $CL$ $HC$ fail to be weakly unchaining. See Theorem \ref{Th: weakly}, Corollary \ref{Cor: weakly} and Example \ref{Ejp: AL-CL}.

We also define a minimal condition of two subsets chained by single points that should be detected. We say that a $HC$ method is \textit{$\alpha$-bridge-unchaining} if it is able to separate two blocks in that situation. $SL^*(\alpha)$ is proved to be more sensitive than that. It also detects some classes of chaining through smaller blocks as it is proved in Proposition \ref{Prop: moderate_2}. In particular, $SL^*(\alpha)$ is $\alpha$-bridge-unchaining. See Corollary \ref{Cor: unchaining}.

The structure of the paper is the following:

Section \ref{Section: background} contains the basic definitions and notation involved. It may be skipped by the experts. In Section \ref{Section: HC} we recall some well known hierachical clustering methods and include some different ways to formulate them. 

In Section \ref{Section: Modified SL} we first introduce the idea of what are we considering undesired chaining effect. Formal definitions are left to the last sections to enhance readability. In this section we are just trying to give the reader some notion of what is $SL(\alpha)$ trying to detect. Then, we present $SL(\alpha)$. We include a short explanation of the role of each step of the method and check it on a few examples. Section \ref{Section: SL^*} deals with the problem of chaining through smaller blocks. Again, we leave formal definitions for the last sections and we only discuss the intuitive idea. Then, we introduce a further step in the algorithm to define $SL^*(\alpha)$ and check it on some examples.


Section \ref{Section: Unchaining properties} studies the unchaining properties of $SL(\alpha)$. First we fix the theoretical background to study the chaining effect. We define the concepts of \textit{chained subsets} and subsets \textit{chained by a single edge}. We define the property of being \textit{strongly chaining}  for $HC$ methods which are extremely sensitive to the chaining effect. We prove that $SL$ is strongly chaining while $AL$ and $CL$ $HC$ are not. Then, we say that a $HC$ method is \textit{weakly unchaining} if it is at least capable of detecting a certain clustering when the subsets are chained by a single edge.  We prove that $SL(\alpha)$ is weakly unchaining while other methods, partially sensitive to the chaining effect as $AL$ and $CL$, are not. We also compare the results obtained by our method and the results offered by the standard clustering algorithm DBSCAN. 

Section \ref{Section: Unchaining properties 2} studies the unchaining properties of $SL^*(\alpha)$. We define the concept of subsets \textit{chained through smaller blocks}. We say that a $HC$ method is \textit{completely chaining} if it is strongly chaining and also unable to detect clusters if they are chained through smaller blocks. We prove that $SL$ is completely chaining. Then, we define the property of being \emph{$\alpha$-bridge-unchaining} for algorithms which are able to detect, at least, some type of subsets chained through single points. We prove that $SL^*(\alpha)$ is  $\alpha$-bridge-unchaining. We compare this case, also, with the treatment of the same input by DBSCAN.

Section \ref{Section: Conclusions} includes a short discussion about the main advantages of the methods defined and some comments about future research. 






\section{Background and notation}\label{Section: background}

A dendrogram over a finite set is a nested family of partitions. This is usually represented as a rooted tree. 

Let $\mathcal{P}(X)$ denote the collection of all partitions of a finite set $X=\{x_1,...,x_n\}$. Then, a dendrogram can also be described as a map $\theta\co [0,\infty)\to \mathcal{P}(X)$ such that:
\begin{itemize}
	\item[1.] $\theta(0)=\{\{x_1\},\{x_2\},...,\{x_n\}\}$,
	\item[2.] there exists $T$ such that $\theta(t)=X$ for every $t\geq T$,
	\item[3.] if $r\leq s$ then $\theta(r)$ refines $\theta(s)$,
	\item[4.] for all $r$ there exists $\varepsilon >0$ such that $\theta(r)=\theta(t)$ for $t\in [r,r+\varepsilon]$.
\end{itemize}

Notice that conditions 2 and 4 imply that there exist $t_0<t_1<...<t_m$ such that $\theta(r)=\theta(t_{i-1})$ for every $r\in [t_{i-1},t_i)$, $i=0,1,...,m$ and $\theta(r)=\theta(t_{m})=\{X\}$ for every $r\in [t_m,\infty)$.

For any partition $\{B_1,...,B_k\}\in \mathcal{P}(X)$, the subsets $B_i$ are called \textit{blocks}.

Let $\mathcal{D}(X)$ denote the collection of all possible dendrograms over a finite set $X$. Given some $\theta \in \mathcal{D}(X)$, let us denote $\theta(t)=\{B_1^t,...,B_{k(t)}^t\}$. Therefore, the nested family of partitions is given by the corresponding partitions at $t_0,...,t_m$, this is, $\{B_1^{t_i},...,B_{k(t_i)}^{t_i}\}$, $i=0,...,m$.

An \emph{ultrametric space} is a metric space $(X,d)$ such that 
$d(x,y)\leq \max \{d(x,z),d(z,y)\}$
for all $x,y,z\in X$. Given a finite metric space $X$ let $\mathcal{U}(X)$ denote the set of all ultrametrics over $X$.

There is a well known equivalence between trees and ultrametrics. See \cite{Hug} and \cite{M-M} for a complete exposition of how to build categorical equivalences between them. In particular, this may be translated into an equivalence between dendrograms and ultrametrics:

Thus, a hierarchical clustering method $\mathfrak{T}$ can be presented as an algorithm whose output is a a dendrogram or an ultrametric space.  Let  $\mathfrak{T}_{\mathcal{D}}(X,d)$ denote the dendrogram obtained by applying $\mathfrak{T}$ to a metric space $(X,d)$ and $\mathfrak{T}_{\mathcal{U}}(X,d)$ denote the corresponding ultrametric space.

Let us define the map $\eta \co \mathcal{D}(X) \to \mathcal{U}(X)$ as follows:

Given a dendrogram $\theta\in \mathcal{D}(X)$, let $\eta(\theta)=u_\theta$ be such that $u_\theta(x,x')=\min\{r\geq 0 \, | \, x,x' \mbox{ belong to the same block of } \theta(r)\}$.

\begin{prop}\cite[Theorem 9]{CM} $\eta$ is a bijection such that $\eta \circ \mathfrak{T}_{\mathcal{D}}=\mathfrak{T}_{\mathcal{U}}$.
\end{prop}

\textbf{Notation}: For any $HC$ method $\mathfrak{T}$ and any finite metric space $(X,d)$, let us denote $\mathfrak{T}_{\mathcal{D}}(X,d)=\theta_X$ and $\mathfrak{T}_{\mathcal{U}}(X,d)=(X,u_X)$. If there is no need to distinguish the metric space we shall just write $\mathfrak{T}_{\mathcal{D}}(X,d)=\theta$ and $\mathfrak{T}_{\mathcal{U}}=u$.

\section{Hierarchical clustering methods}\label{Section: HC}

Let us recall the definition of some well-known hierarchical clustering methods. We include here the description of \textit{single linkage} based on its $t$-connected components. Also the recursive description of \textit{single linkage}, \textit{complete linkage} and \textit{average linkage} as presented in \cite{CM}. We introduce also an alternative description of these methods, based in the recursive one. In our description we define a graph, $G_R^\ell$, which will be the key to build our new method, $SL(\alpha)$. We think that this approach might be useful to define other algorithms which might be better adapted to other specific problems. 

An $\varepsilon$-\emph{chain} is a finite sequence of points $x_0, ..., x_N$ that are separated by distances less or equal than $\varepsilon$: $d(x_i, x_{i+1}) < \varepsilon$. Two points are $\varepsilon$-\emph{connected} if there is an $\varepsilon$-chain joining them. Any two points in an $\varepsilon$-\emph{connected set} can be linked by an $\varepsilon$-chain. An $\varepsilon$-\emph{component} is a maximal $\varepsilon$-connected subset.

\begin{figure}[ht]
\centering
\includegraphics[scale=0.3]{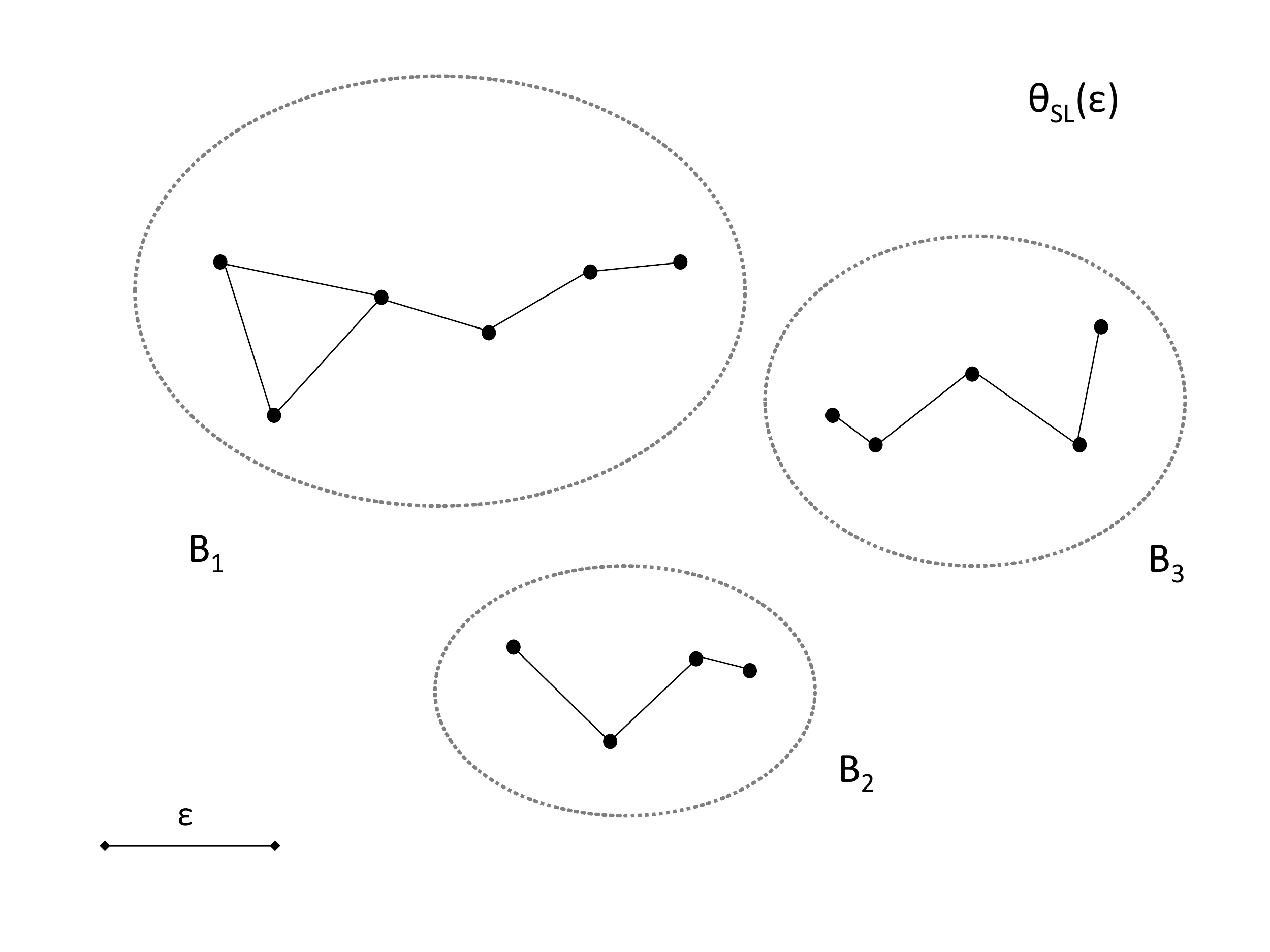}
\caption{$\theta_{SL}(\varepsilon)$ is the partition of $X$ in its $\varepsilon$-components.}
\label{Ejp epsilon-chains}
\end{figure}

Clearly, given a metric space and any $\varepsilon>0$, there is a partition of $X$ in its $\varepsilon$-components $\{C_1^\varepsilon,...,C_{k(\varepsilon)}^\varepsilon\}$.

Let $X$ be a finite metric set. The \textit{single linkage} $HC$ is defined by the map $\theta_{SL}\co [0,\infty)\to \mathcal{P}(X)$ such that $\theta_{SL}(t)$ is the partition of $X$ in its $t$-components. See Figure \ref{Ejp epsilon-chains}.

For other uses of $\varepsilon$-connectedness on computational topology see \cite{R} and \cite{M-C-L}. 



In \cite{CM} there is also an alternative formulation of $SL$ $HC$. In fact, the authors use a recursive procedure to redefine $SL$ $HC$, average linkage ($AL$) and complete linkage ($CL$) hierarchical clustering. The main advantage of this procedure is that it allows to merge more than two clusters at the same time. Therefore,  $AL$ and $CL$ $HC$ can be made \textit{permutation invariant}, meaning that the result of the hierarchical clustering does not depend on the order in which the points are introduced in the algorithm.  We reproduce here, for completeness, their formulation. 

Let $(X,d)$ be a finite metric space where $X=\{x_1,...,x_n\}$ and let $L$ denote a family of linkage functions on $X$:
\[L:=\{\ell \colon \mathcal{C}(X)\times \mathcal{C}(X) \to \br^+ \, | \, \ell \mbox{ is bounded and non-negative } \}\]
where $\mathcal{C}(X)$ denotes the collection of all non-empty subsets of $X$.

Some standard choices for $\ell$ are:

\begin{itemize}
	\item Single linkage: $\ell^{SL}(B,B')=\min_{(x,x')\in B\times B'}d(x,x')$ 
	\item Complete linkage: $\ell^{CL}(B,B')=\max_{(x,x')\in B\times B'}d(x,x')$
	\item Average linkage: $\ell^{AL}(B,B')=\frac{\sum_{(x,x')\in B\times B'}d(x,x')}{\#(B)\cdot \#(B')}$ where $\#(X)$ denotes the cardinality of the set $X$.
\end{itemize} 

Fix some linkage function $\ell\in L$. Then, the recursive formulation is as follows

\begin{itemize}
	\item[1.] For each $R>0$ consider the equivalence relation $\sim_{\ell,R}$ on blocks of a partition $\Pi\in \mathcal{P}(X)$, given by $B\sim_{\ell,R}B'$ if and only if there is a sequence of blocks $B=B_1,...,B_s=B'$ in $\Pi$ with $\ell(B_k,B_{k+1})\leq R$ for $k=1,...,s-1$.
	\item[2.] Consider the sequences $R_0,R_1,R_2,... \in [0,\infty)$ and $\Theta_0, \Theta_1, \Theta_2,... \in \mathcal{P}(X)$ given by  $R_0=0$, $\Theta_0:=\{x_1,...,x_n\}$, and recursively for $i\geq 1$ by $\Theta_{i}=\frac{\Theta_{i-1}}{\sim_{\ell,R_{i}}}$ where $R_{i}:=\min\{\ell(B,B')\, | \, B,B'\in \Theta_{i-1}, \ B\neq B'\}$ until $\Theta_{i}=\{X\}$.
	\item[3.] Finally, let $\theta^\ell\colon [0,\infty)\to \mathcal{P}(X)$ be such that $\theta^\ell(r):=\Theta_{i(r)}$ with $i(r):=\max\{i\, | \, R_i\leq r\}$.
\end{itemize}

\begin{obs}\label{Remark: components} We can also reformulate the recursive algorithm as follows.

\begin{itemize}
	\item[1.] Let $\Theta_0:=\{x_1,...,x_n\}$ and $R_0=0$.
	\item[2.] For every $i\geq 1$, while $\Theta_{i-1}\neq \{X\}$, let $R_{i}:=\min\{\ell(B,B')\, | \, B,B'\in \Theta_{i-1}, \ B\neq B'\}$. Then, let $G_{R_{i}}^\ell$ be a graph whose vertices are the blocks of $\Theta_{i-1}$ and such that there is an edge joining $B$ and $B'$ if and only if $\ell(B,B')\leq R_{i}$. 
	\item[3.] Consider the equivalence relation $B\sim_{\ell,R} B'$ if and only if $B,B'$ are in the same connected component of $G_R^\ell$. Then, $\Theta_{i}=\frac{\Theta_{i-1}}{\sim_{\ell,R_{i}}}$.
	\item[4.] Finally, let $\theta^\ell\colon [0,\infty)\to \mathcal{P}(X)$ be such that $\theta^\ell(r):=\Theta_{i(r)}$ with $i(r):=\max\{i\, | \, R_i\leq r\}$.
\end{itemize} 
\end{obs}

This formulation allows us to consider other properties of the graph $G_R^\ell$, along with the connected components, to define the equivalence relation $B\sim_{\ell,R} B'$. This means that we may introduce further conditions to merge the blocks and, this way, reduce some indesired results as the chaining effect. See also \cite{M}.

\section{$\alpha$-unchaining single linkage hierarchical clustering: $SL(\alpha)$}\label{Section: Modified SL}

The chaining effect is usually mentioned as one of the problems to solve in clustering. However, there are different approaches to define ``chaining effects''.

In \cite{CM}, the authors refer to the chaining effect from \cite{LW} which is the one defined by Williams, Lambert and Lance in 
\cite{WLL}. This version of the ``chaining effect'' takes account on the tendency of a group to merge with single points or small groups rather than with other groups of comparable size. Thus, Williams, Lambert and Lance study and measure it by comparing the cardinality of the groups.

Herein, we are focusing on another aspect. We want to deal with the tendency of two clusters to be merged when the minimal distance between them is small independently of their distribution. This can be a problem when the clusters have dense cores distant apart. See figure \ref{Ejp Chained edge}. This is, typically, the chaining effect one finds in $SL$ $HC$.

\begin{figure}[ht]
\centering
\includegraphics[scale=0.4]{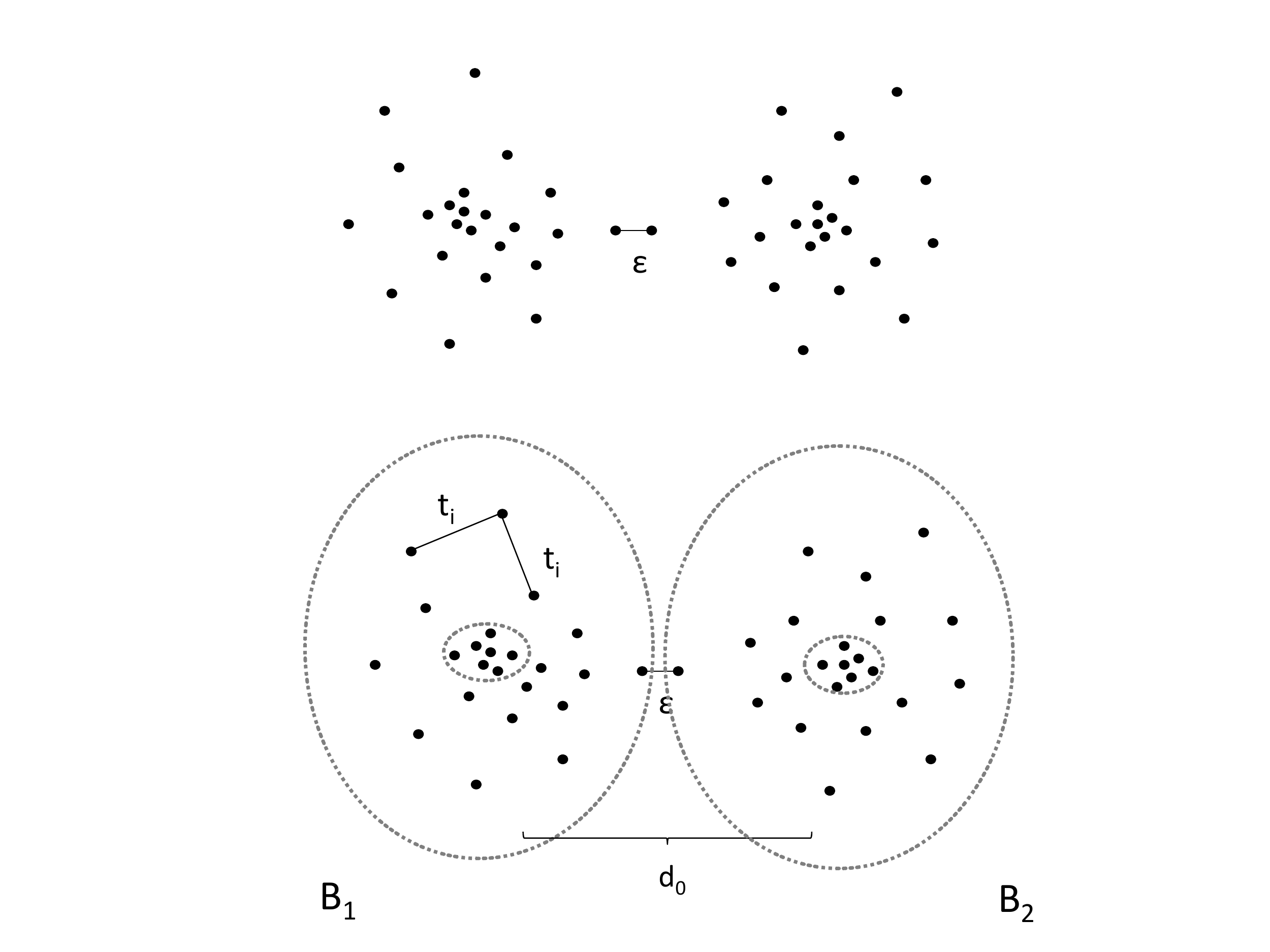}
\caption{The minimal distance between the blocks $B_1$ and $B_2$ is $\varepsilon$. The clustering $\{B_1,B_2\}$ would not be detected by $SL$ $HC$.}
\label{Ejp Chained edge}
\end{figure}

Given a finite metric space $(X,d)$, let $F_t(X,d)$ be the Rips (or Vietoris-Rips) complex of $(X,d)$. Let us recall that the Rips complex of a metric space $(X,d)$ is a simplicial complex whose vertices are the points of $X$ and $[v_0,...,v_k]$ is a simplex of $F_t(X,d)$ if and only if $d(v_i,v_j)\leq t$ for every $i,j$. Given any subset $Y\subset X$, by $F_t(Y)$ we refer to the subcomplex of $F_t(X)$ defined by the vertices in $Y$. A simplex $[v_0,...,v_k]$ has dimension $k$. The dimension of a simplicial complex is the maximal dimension of its simplices.

Notice that densely packed points produce high-dimensional simplices in the Rips complex.

\begin{figure}[ht]
\centering
\includegraphics[scale=0.4]{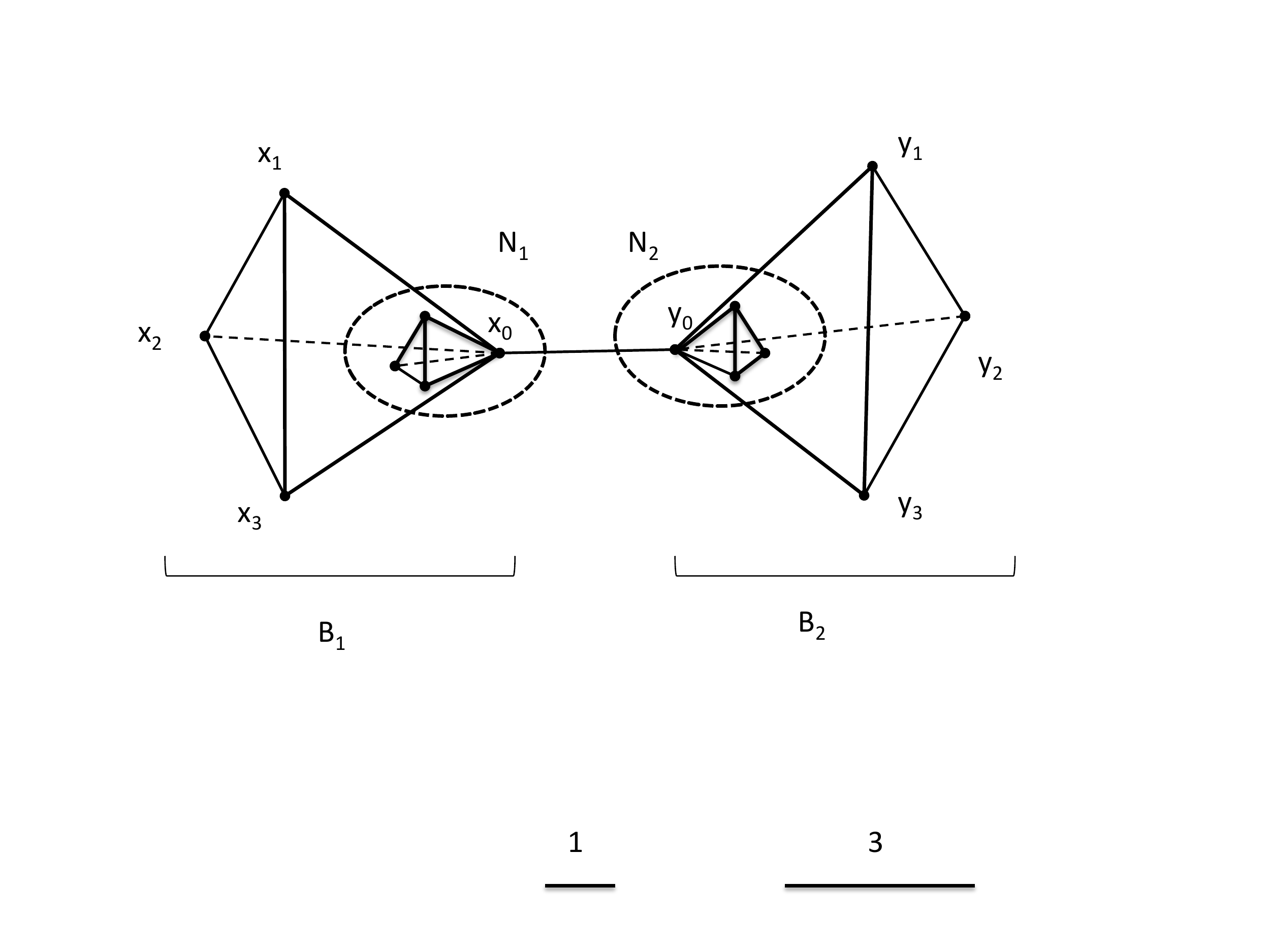}
\caption{The clustering $\{B_1,B_2\}$ is detected by $SL(\alpha)$ for $\alpha<3$.}
\label{Ejp Chain_1b}
\end{figure}

\begin{ejp}\label{Ejp Chain1} Let $(X,d)$ be the graph from Figure \ref{Ejp Chain_1b}. The edges in $N_1,N_2$ have length 1 and the rest have length 3. The distances between vertices are measured as the minimal length of a path joining them. 

Consider $F_{1}(X,d)$. Then, the vertices of $N_1$ and $N_2$ define 3-dimensional simplices. For any vertex $v\in X\backslash \{N_1\cup N_2\}$, there is no vertex $w$ with $d(v,w)\leq 1$. Therefore, they define 0-dimensional simplices and they are not part of any 1-dimensional simplex of $F_{1}(X,d)$.  
\end{ejp}

We define a modified single linkage hierarchical clustering method, $SL(\alpha)$, on the basis of $SL$ introducing a parameter $\alpha\in \bn$. This method allows us to take into account density without having to provide any additional input to the algorithm apart from $\alpha$ and the distances between the points. 

Let $(X,d)$ be a finite metric space with $X=\{x_1,...,x_n\}$. 

Notice that in section \ref{Section: HC} the recursive definition of $SL$, $CL$ and $AL$ used the distances, $R_i$, between the blocks from the previous step. For technical reasons, to define our method we need to use the ordered set of distances in the data set, $(D,<)$.

Let $d_{ij}:=d(x_i,x_j)$ and  $D:=\{t_i \, : \, 0\leq i \leq m\}=\{d_{ij} \, : \, 1\leq i, j \leq n\}$ with $t_i<t_j$ $\forall \, i<j$ where ``<'' denotes the order of the real numbers. Clearly, $t_0=0$.


Let the dendrogram defined by $SL(\alpha)$, $\mathfrak{T}^{SL(\alpha)}_\mathcal{D}(X,d)=\theta_{X,\alpha}$ or simply $\theta_\alpha$, be as follows:

\begin{itemize}

	\item[1)] Let $\theta_\alpha(0):=\{\{x_1\},...,\{x_n\}\}$ and $\theta_\alpha(t):=\theta_\alpha(0)$ $\forall t< t_1$. Now, for every $i$, given $\theta_\alpha[t_{i-1},t_i)=\theta_\alpha(t_{i-1})=\{B_1,...,B_m\}$, we define recursively  $\theta_\alpha$ on the interval $[t_i,t_{i+1})$ as follows: 

	\item[2)] Let $G_\alpha^{t_i}$ be a graph with vertices $\mathcal{V}(G_\alpha^{t_i}):=\{B_1,...,B_m\}$ and edges $\mathcal{E}(G_\alpha^{t_i}):=\{B_j,B_k\}$ such that the following conditions hold:
		\begin{itemize}
			\item[i)] $\min\{d(x,y)\, | \, x\in B_j,\ y\in B_k\}\leq t_i$.
			\item[ii)] there is a simplex $\Delta \in F_{t_i}(B_j\cup  B_k)$  such that $\Delta \cap B_j\neq \emptyset$, $\Delta \cap B_k\neq \emptyset$ and $\alpha \cdot dim(\Delta)\geq \min\{dim (F_{t_i}(B_j)), dim (F_{t_i}(B_k))\}$. 
		\end{itemize}


	\item[3)] Let us define a relation, $\sim_{t_i,\alpha}$ as follows.
		
Let $B_j\sim_{t_i,\alpha}B_k$ if $B_j,B_k$ belong to the same connected component of the graph $G_\alpha^{t_i}$. Then, $\sim_{t_i,\alpha}$ induces an equivalence relation.

	\item[4)] For every  $t\in [t_i,t_{i+1})$, $\theta_\alpha(t):=\theta_\alpha(t_{i-1})/\sim_{t_i,\alpha}$.
\end{itemize}

Condition $i)$ is the condition used in $SL$ $HC$ to define the graph. See Remark \ref{Remark: components}.

Condition $ii)$ is used to account for the chaining effect between two adjacent blocks. Suppose we have two adjacent blocks, densely packed, which are close to each other as sets but whose dense cores are distant apart as in Figure \ref{Ejp Chained edge}. Then, the dense cores will produce high dimensional simplices in the Rips complex while the connection between the blocks might be a low dimensional simplex. In this case,  condition $ii)$ will not be satisfied and the edge between the corresponding blocks in  $G_\alpha^{t_i}$ is not defined (although condition $i)$ holds). Hence, these blocks will not be merged.

\begin{obs}\label{Remark: necessary chain} Notice that if two points $x,x'$ belong to the same block  of $\theta_{\alpha}(t_i)$  then, necessarily, there exists a $t_i$-chain, $x=x_0,x_1,...,x_n=x'$ joining them. In particular, if $x_j\in B_j\in \theta_\alpha(t_{i-1})$, $j=0,...,n$, the corresponding edges $\{B_{j-1},B_j\}$, $j=1,n$, satisfies condition $ii)$. This is immediate by construction. 
\end{obs}

\begin{figure}[ht]
\centering
\includegraphics[scale=0.4]{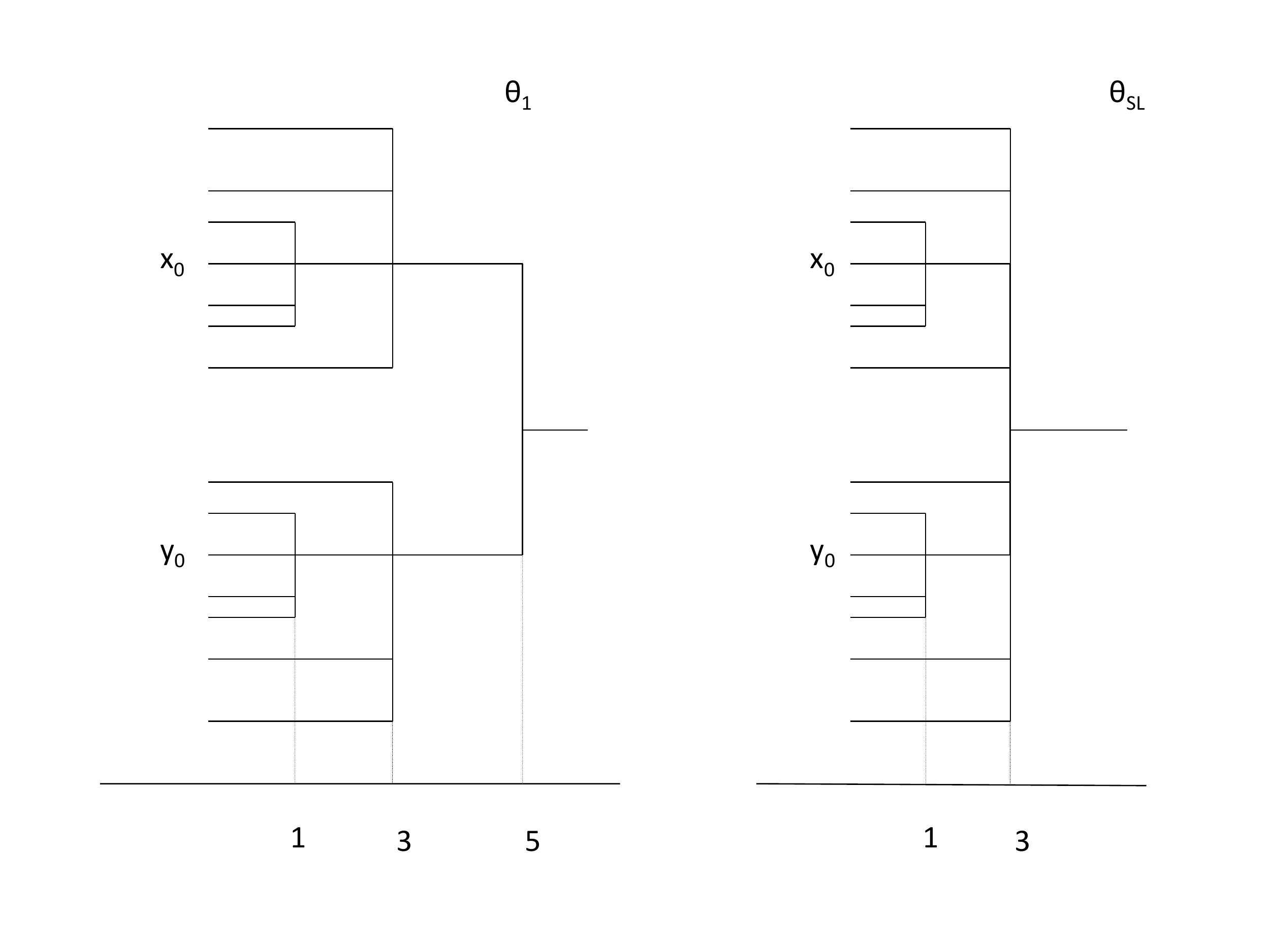}
\caption{Dendrogram produced by $SL(1)$ for the graph in Figure \ref{Ejp Chain_1b} compared with the corresponding dendrogram obtained applying $SL$ or $SL(\alpha)$ with $\alpha >3$}
\label{Chain_dendrogram_1}
\end{figure}

\begin{ejp}\label{Ejp: unchain 1} Let $(X,d)$ be the graph from Figure \ref{Ejp Chain_1b} and let $\alpha=1$. Notice that there are eight $1$-components, six of them are singletons and two of them, $N_1$, $N_2$, with $\#(N_1)=\#(N_2)=4$. Furthermore, $x_0\in N_1$, $y_0\in N_2$ and $dim\, F_{1}(N_s)=3$ for $s=1,2$. 
Then, let us check that applying $SL(1)$ on $(X,d)$ the clustering $\{B_1,B_2\}$ is detected. 

Let $\mathfrak{T}^{SL(1)}_\mathcal{D}(X,d)=\theta_1$. $\theta_1(t)=\{\{x_0\},...,\{x_{6}\},\{y_0\},...,\{y_{6}\}\}$ if $t<1$. It $1\leq t<3$, $\theta_1(t)=\{\{x_1\},\{x_2\},\{x_3\},N_1,N_2,\{y_1\},\{y_2\},\{y_3\}\}\}$. 

For $t=3$, there are edges in $G_1^3$ between every pair of clusters in $N_1, \{x_1\},\{x_2\},\{x_3\}$. Similarly, there are edges between every pair of clusters $N_2,\{y_1\},\{y_2\},\{y_3\}$. $F_3(N_1)$ and $F_3(N_2)$ have dimension 3 while the unique simplex in $F_3(X)$ intersecting both $N_1,N_2$ has dimension 1. Therefore,  condition $ii)$ induces a separation of the blocks $N_1, N_2$. Thus, $\theta_1(3)=\{B_1,B_2\}$.

For $t=4$, $dim(F_4(B_1))=6=dim(F_4(B_2))$ while the maximal dimension of a simplex intersecting both clusters is $4$. Then, by $ii)$, there is no edge in $G_1^4$ joining $B_1,B_2$. 

For $t=5$, $dim(F_5(B_1))=6=dim(F_5(B_2))$. Since the diameter of $N_1\cup N_2$ is 5, these vertices define a simplex $\Delta$ in $F_5(X)$ such that $dim(\Delta)=7$. Clearly, this simplex intersects $B_1$ and $B_2$.  Hence, $\theta_1(5)=\{X\}$. Therefore, the dendrogram obtained, $\theta_1$, is the one from Figure \ref{Chain_dendrogram_1}.

Modifying the parameter $\alpha$ we can adjust the method to be more or less sensitive to the chaining effect. Increasing $\alpha $ we would need higher dimensions in $F_{3}(N_s)$ to unchain blocks by condition $ii)$. Suppose $\alpha \geq 3$. In this case, $B_1,B_2$ would be joined by an edge in $G_\alpha^3$.  Thus, for $\alpha \geq 3$,  $\theta_\alpha(t)=\{\{x_0\},...,\{x_{6}\},\{y_0\},...,\{y_{6}\}\}$ if $t<1$, $\theta_\alpha(t)=\{\{x_1\},\{x_2\},\{x_3\},N_1,N_2,\{y_1\},\{y_2\},\{y_3\}\}\}$ if $1\leq t <3$ and $\theta_\alpha(t)=\{X\}$ if $t\geq 3$.
\end{ejp}

\begin{figure}[ht]
\centering
\includegraphics[scale=0.4]{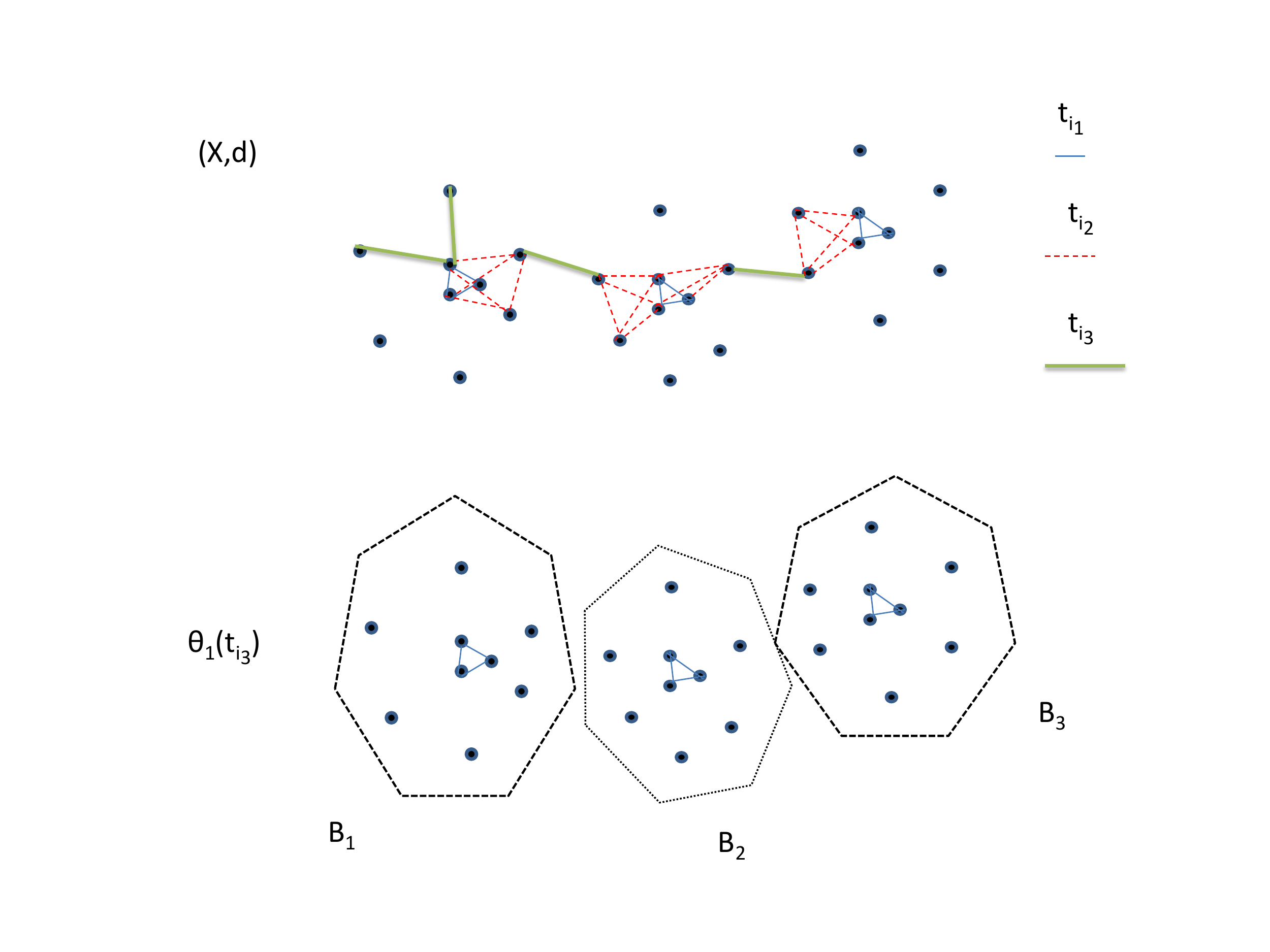}
\caption{For the set of points in the figure appears to be a natural clustering $\{B_1, B_2,B_3\}$.}
\label{Ejp SL}
\end{figure}

\begin{figure}[ht]
\centering
\includegraphics[scale=0.4]{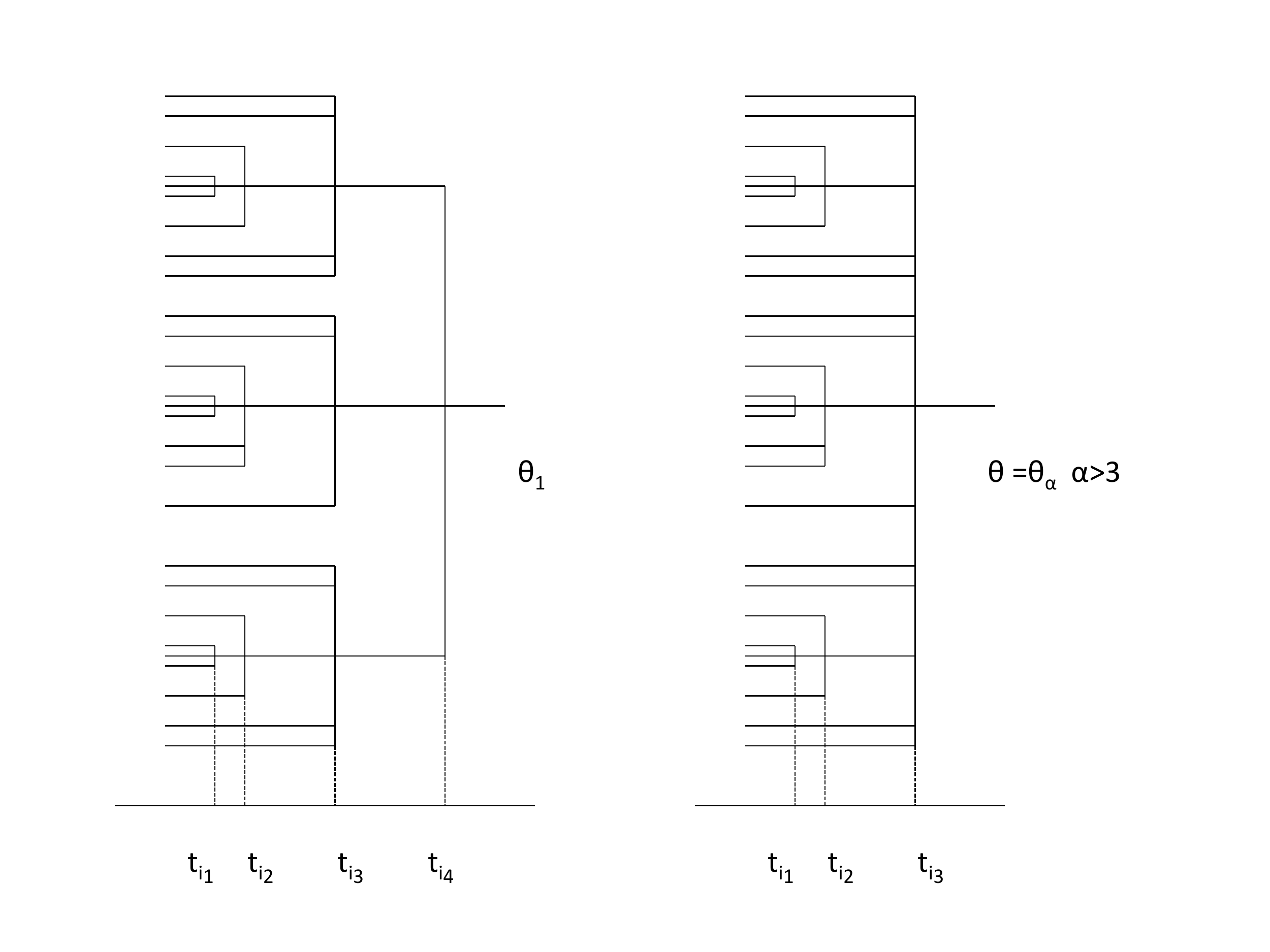}
\caption{Dendrograms $\theta_1$ and $\theta_\alpha$ with $\alpha>3$ for Example \ref{Ejp SL}.}
\label{Ejp SL(1)}
\end{figure}

\begin{ejp}\label{Ejp: unchain 2} Consider the set represented in Figure \ref{Ejp SL}. Let us consider three distances, $t_{i_1}<t_{i_2}<t_{i_3}$ which are represented, respectively, by a short segment, a dots line and a thick long segment. Let us assume that the sets $B_1,B_2,B_3$ are $(t_{i_3})$-connected and that $d(B_k,B_{k+1})= t_{i_3}$, $k=1,2$. Also, we can see that there exist 3-dimensional simplices in $F_{t_{i_2}}(X)$ inside $B_1$, $B_2$ and $B_3$. In Figure  \ref{Ejp SL(1)}, we represent the corresponding dendrograms for $SL(1)$ and $SL$. (Notice that, in this case, $SL=SL(\alpha)$ for any $\alpha >3$.) 

It is clear that, $SL$ $HC$ generates a dendrogram where it is impossible to detect the clustering $\{B_1,B_2,B_3\}$ because of the chaining effect. Introducing the parameter $\alpha =1$, in this example, we obtain a hierarchical clustering  which is consistent with the distribution of the sample.
\end{ejp}

\begin{figure}[ht]
\centering
\includegraphics[scale=0.3]{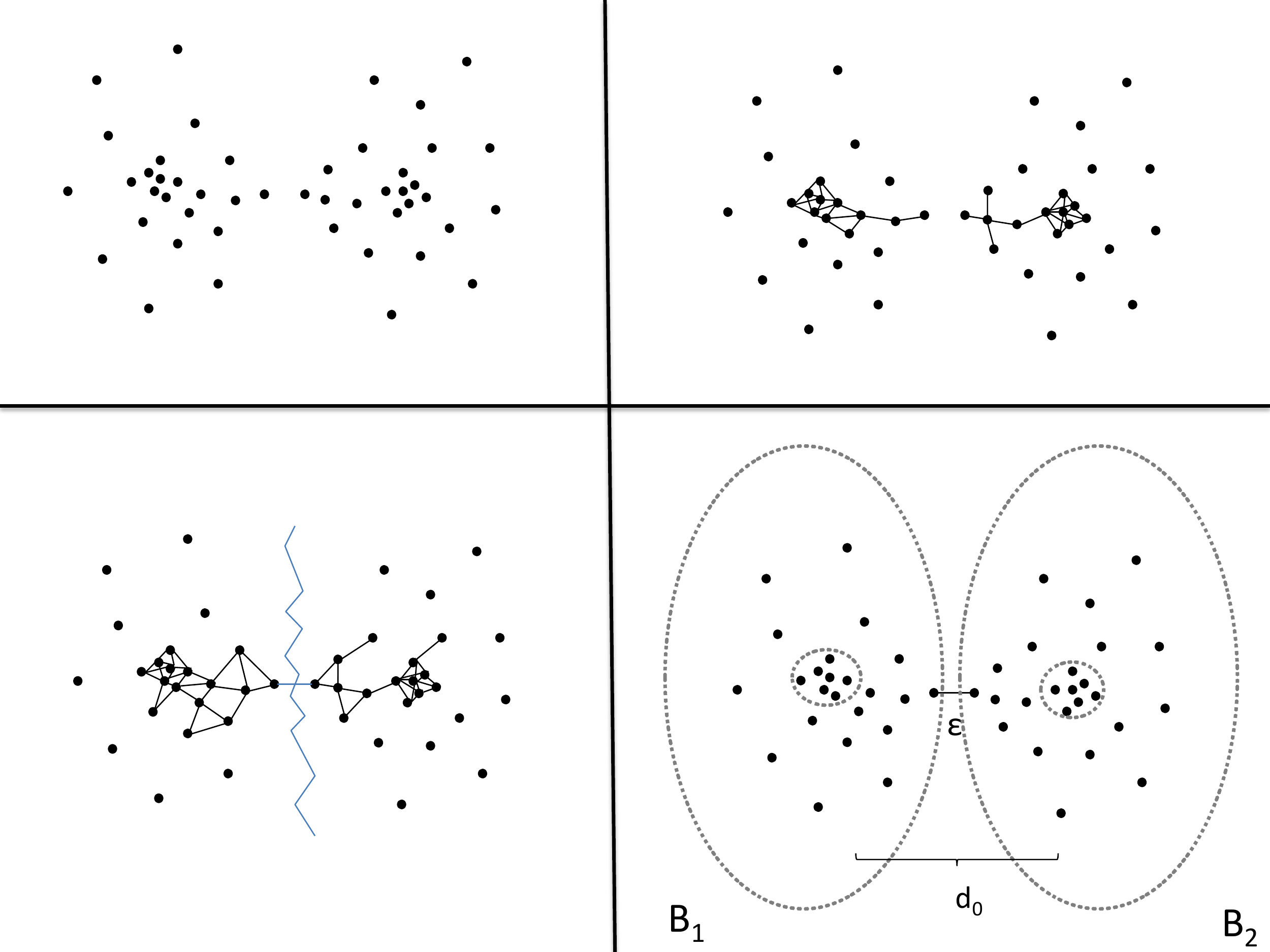}
\caption{The clustering $\{B_1,B_2\}$ is detected by $SL(3)$.}
\label{Ejp Chained_edge_2}
\end{figure}

\begin{ejp}\label{Ejp: chained_edge_2} Let $X$ be the set from Figure \ref{Ejp Chained_edge_2}. Suppose $\varepsilon=t_i$ and let us see what happens in the application of the algorithm $SL(3)$ until $t=t_{i-1}$. 

In the second square of the figure we can see the 1-dimensional skeleton of the Rips complex $F_{t_{i-1}}(X)$. As it is shown in the figure, there is no edge joining $B_1$ and $B_2$ yet and the dense cores inside $B_1,B_2$ produce high dimensional simplices in $F_{t_{i-1}}(X)$. Condition $ii)$ has not been applied yet and, therefore, the blocks are just the $t_{i-1}$-components. Let us call $N_1$, $N_2$ the corresponding blocks in $\theta_{3}$ defined by the nontrivial $t_{i-1}$components.

For $t_i=\varepsilon$ there is a single edge joining $N_1$ and $N_2$ as we can see in the third square. In particular, $\ell^{SL}(N_1,N_2)=\varepsilon$. However, the dimension of $F_{t_{i}}(N_s)$, $s=1,2$ is greater that 3. Then,  by condition $ii)$, there is no edge in the graph $G^{t_i}_3$ joining $N_1$ and $N_2$ and $\theta_{t_i}(X)$ refines the clustering $\{B_1,B_2\}$. In fact, the clustering from $\theta_3(t_i)$ is given by the connected components in the third square of the figure when the edge of length $\varepsilon$ joining $B_1,B_2$ is eliminated. 

Let $t_k=\min\{t \, | \, B_1,B_2 \mbox{ are $t$-connected}\}$. In the example, by condition $ii)$, for every $t_i< t_j\leq t_k$ and for any pair of blocks $C_1,C_2\in \theta_3(t_{j-1})$ with $C_1\in B_1$ and $C_2\in B_2$, there is no edge between them in $G^{t_j}_3$. Thus, $C_1,C_2$ are not merged in $\theta(t_j)$. Therefore, $\theta(t_k)=\{B_1,B_2\}$.
\end{ejp}

It may be noticed that $SL(\alpha)$ does not detect the possible clustering $\{B_1,B_2\}$ in the graph represented in Figure \ref{Ejp Chain_1}
. See Example \ref{Ejp: 1-1-chained} below. This illustrates the fact that our method does not consider directly the distribution of the points for a certain $t_i$. Instead of that, it focuses on the relations between the blocks from $\theta(t_{i-1})$.

\begin{figure}[ht]
\centering
\includegraphics[scale=0.3]{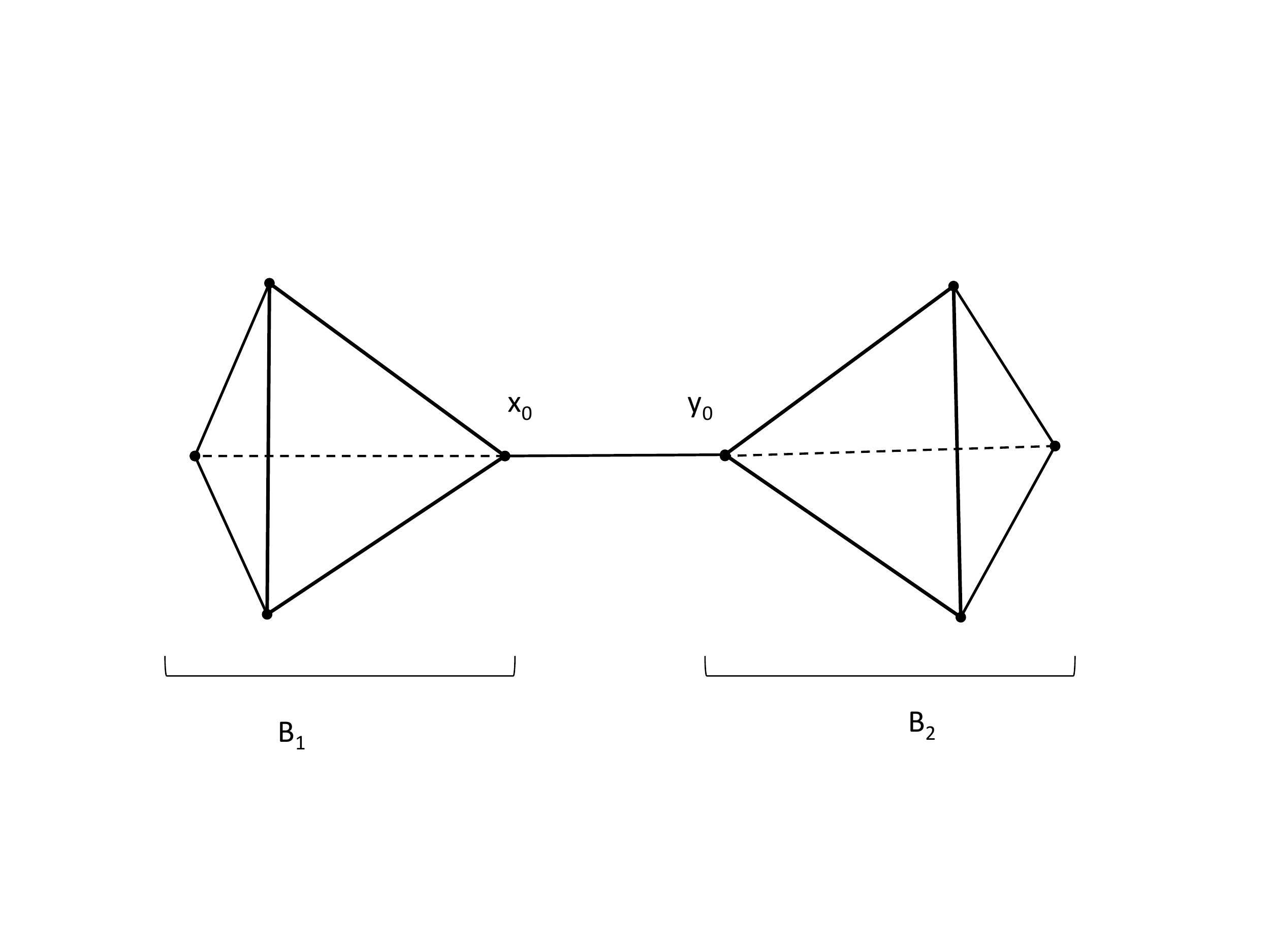}
\caption{The blocks $B_1,B_2$ have no dense cores. $SL(\alpha)$ does not separate them.}
\label{Ejp Chain_1}
\end{figure}

\begin{ejp}\label{Ejp: 1-1-chained}  Let $X_1$ be the graph from Figure \ref{Ejp Chain_1} where every edge has length $1$. 
Let us fix $\alpha=1$.

If $t<1$, $\theta_1(t)=\{\{x_0\},...,\{x_3\},\{y_0\},...\{y_3\}\}$. If $t=1$, in $ii)$ we consider the dimension of the complexes defined from the blocks in $t_0$ which are singletons. Therefore, every edge in the graph defines an edge in $G_1^{1}$.

Since $G_1^{1}$ is connected $\theta_1(1)=\{X_1\}$. 

For $\alpha>1$ also $SL(\alpha)=SL(1)=SL$.
\end{ejp}

\section{Chaining through smaller blocks: $SL^*(\alpha)$}\label{Section: SL^*}

\begin{figure}[ht]
\centering
\includegraphics[scale=0.4]{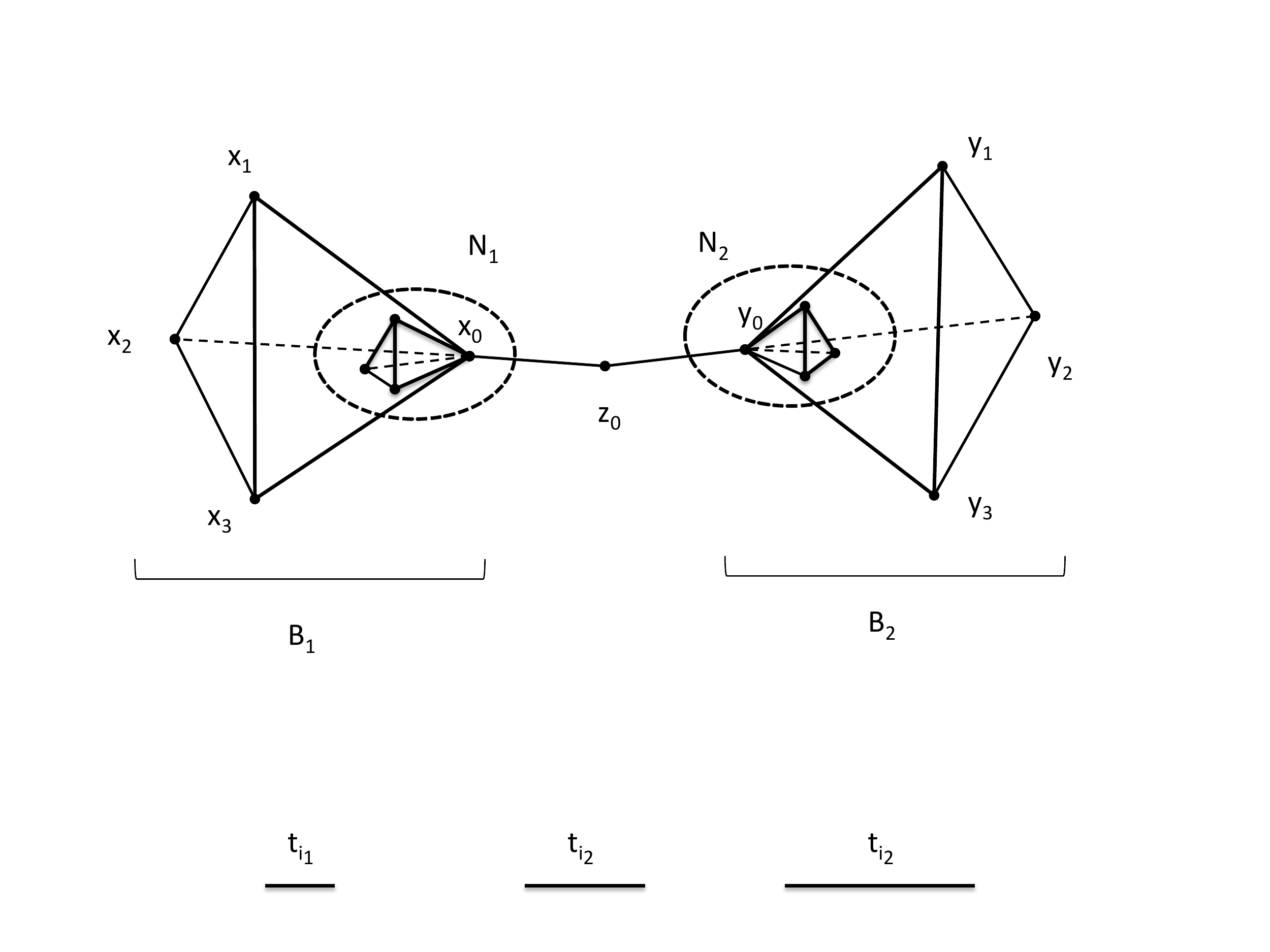}
\caption{Components $B_1$ and $B_2$ are chained through the smaller block $\{z_0\}$.}
\label{Ejp Chain_2b}
\end{figure}

The method $SL(\alpha)$ is defined to prevent two adjacent blocks with dense cores to be chained too soon when the minimal distance between them is small. Thus, condition $ii)$ considers the dimension of the Rips complex restricted to both blocks. However, any cluster $B\in \theta_\alpha(t_{i-1})$ and any isolated point $\{z\}\in \theta_\alpha(t_{i-1})$ such that $d(B,z)\leq t_i$ are going to be merged in $\theta_\alpha(t_i)$. Therefore, if two clusters are at a certain distance $\varepsilon$ from a single point then both blocks will be merged with this point in $\theta_\alpha(\varepsilon)$. Consequently, those clusters will be merged together in $\theta_\alpha(\varepsilon)$.

\begin{ejp} Consider the graph represented in Figure \ref{Ejp Chain_2b}.

Suppose that the edges in $N_1$,$N_2$ have length $1$, the edges $\{x_0,z_0\}$ and $\{z_0,y_0\}$ have length $2$ and the rest have length $3$. As we can see in the picture, there are 9 $1$-components, two of them, $N_1$ and $N_2$, have four points and the rest are singletons. The whole space is $3$-connected.

Let us fix any $\alpha\geq 1$. It is trivial to check that, $\theta_\alpha(1)=\{x_1,x_2,x_3,N_1,z_0, N_2,y_1,y_2,y_3\}$. Now, for $t=2$, since $\{z_0\}$ is a single point and $dim(F_2(\{z_0\}))=0$, condition $ii)$ is trivially satisfied. Therefore, there exist edges in $G_\alpha^2$ between $N_1$ and $z_0$ and between $z_0$ and $N_1$. Hence, $\theta_\alpha(2)=\{x_1,x_2,x_3,\{N_1\cup \{z_0\}\cup N_2\},y_1,y_2,y_3\}$.

Similarly, since every block in $\theta_\alpha(2)$ except from $\{N_1\cup z_0\cup N_2\}$ is a single point, condition $ii)$ always holds and $\theta_\alpha(3)=\{X\}$.
\end{ejp}

In general, for any pair of clusters $B_1,B_2\in \theta_\alpha(t_{i-1})$ such that $d(B_1,B_2)\leq t_i$ and $\#(B_2)\leq \alpha$ the dimension of the Rips complex restricted to $B_2$ is at most $\alpha-1$. Therefore, condition $ii)$ will not apply and the clusters will be merged. Thus, for any chain of clusters $B_0,...,B_n$ such that $\ell_{SL}(B_{i-1},B_i)=\varepsilon$ for every $i=1,...,n$ and $\#(B_j)\leq \alpha$ for every $j=1,...,n-1$, $B_0,...,B_n$ are merged together in $\theta_\alpha(\varepsilon)$. We call this chaining effect: \textbf{chaining through smaller blocks}.

Now, we are going to modify the algorithm so that it may distinguish the case when to blocks are chained by isolated points or small blocks. These points or small blocks might be considered as noise in the sample. See, for example, the point $z_0$ in Figure \ref{Ejp Chain_2b}.

To treat this effect we are going to focus in the ``big'' blocks. The selection is done depending on the parameter $\alpha$ (which defines the sensitivity of the whole method to chaining) and on the cardinality of the blocks involved. We use (\ref{Big blocks}) to fix the distinction between big blocks and small blocks.


Let $X=\{x_1,...,x_n\}$. Let $(D,<)=\{t_i \, : 1\leq i\leq n\}$ be the ordered set of distances between points of $X$.

Let the dendrogram defined by $SL^*(\alpha)$, $\mathfrak{T}^{SL^*(\alpha)}_\mathcal{D}(X,d)=\theta^*_{X,\alpha}$ or simply $\theta^*_\alpha$,  be as follows:


\begin{itemize}
	\item[1)] Let $\theta^*_\alpha(0):=\{\{x_1\},...,\{x_n\}\}$ and $\theta^*_\alpha(t):=\theta^*_\alpha(0)$ $\forall t< t_1$. 
	
	Now, given $\theta^*_\alpha[t_{i-1},t_i)=\theta^*(t_{i-1})=\{B_1,...,B_m\}$, we define recursively  $\theta^*_\alpha$ on the interval $[t_i,t_{i+1})$ as follows:

	\item[2)] Let $G_\alpha^{t_i}$ be a graph with vertices $\mathcal{V}(G_\alpha^{t_i}):=\{B_1,...,B_m\}$ and edges $\mathcal{E}(G_\alpha^{t_i}):=\{B_j,B_k\}$ such that the following conditions hold:
		\begin{itemize}
			\item[i)] $\min\{d(x,y)\, | \, x\in B_j,\ y\in B_k\}\leq t_i$.
			\item[ii)] there is a simplex $\Delta \in F_{t_i}(B_j\cup  B_k)$  such that $\Delta \cap B_j\neq \emptyset$, $\Delta \cap B_k\neq \emptyset$ and $\alpha \cdot dim(\Delta)\geq \min\{dim (F_{t_i}(B_j)), dim (F_{t_i}(B_k))\}$. 
		\end{itemize}

By an abuse of the notation, we may write $B$ to refer both to the block of $\theta(t_{i-1})$ and to the vertex of $G_\alpha^{t_i}$.

	\item[3)] Let us define a relation, $\sim_{t_i,\alpha}$ between the blocks as follows.

Let $cc(G_\alpha^{t_i})$ be the set of connected components of the graph $G_\alpha^{t_i}$. Let $A\in cc(G_\alpha^{t_i})$ with $A=\{B_{j_1},...,B_{j_r}\}$.

Let us call \textbf{big blocks} of $A$ those blocks such that
\begin{equation}\label{Big blocks} \alpha \cdot \#(B_{j_k})\geq \max_{1\leq l \leq r}\{ \#(B_{j_l})\}.\end{equation}

The rest of blocks of $A$ are called \textbf{small blocks}.

Let $H_\alpha(A)$ be the subgraph of $A$ whose vertices are the big blocks and $S_\alpha(A)$ be the subgraph of $A$ whose vertices are the small blocks.
	
Then, $B_{j_k}\sim_{t_i,\alpha} B_{j_{k'}}$ if one of the following conditions holds:

		\begin{itemize}
			\item[iii)] $\exists \, C\in cc(H_\alpha(A))$ such that $B_{j_k},B_{j_{k'}}\in C$. 
			\item[iv)] $B_{j_k}\in C \in cc(H_\alpha(A))$, $B_{j_{k'}}\in C'\in  cc(S_\alpha(A))$ and there is no big block in $A\backslash C$ adjacent to any block in $C'$.

		\end{itemize}

Then, $\sim_{t_i,\alpha}$ induces an equivalence relation whose classes are contained in the connected components of $G_\alpha^{t_i}$.

	\item[4)] For every  $t\in [t_i,t_{i+1})$, $\theta^*_\alpha(t):=\theta^*_\alpha(t_{i-1})/\sim_{t_i,\alpha}$.
\end{itemize}

Step $1)$ and $2)$ are the same as in $SL(\alpha)$.

By $iii)$, if two big blocks, $B,B'$, are joined by and edge in $G_\alpha^{t_i}$, then $B\sim_{\alpha,t_i}B'$. Thus, the connected components of big blocks are merged.

By $iv)$, a connected component of small blocks $C'$ is merged with a component of big blocks $C$ if $C$ is the unique component of big blocks adjacent to $C'$. Otherwise, the blocks of $C'$ stay as separated blocks in $\theta^*_\alpha(t_i)$. This can be seen also as follows. By $iv)$, if a small block is connected by chains of small blocks to two different components of big blocks we will consider it as a block apart in $\theta(t_{i})$. See Example \ref{Ejp: Big_small}.

\begin{figure}[ht]
\centering
\includegraphics[scale=0.4]{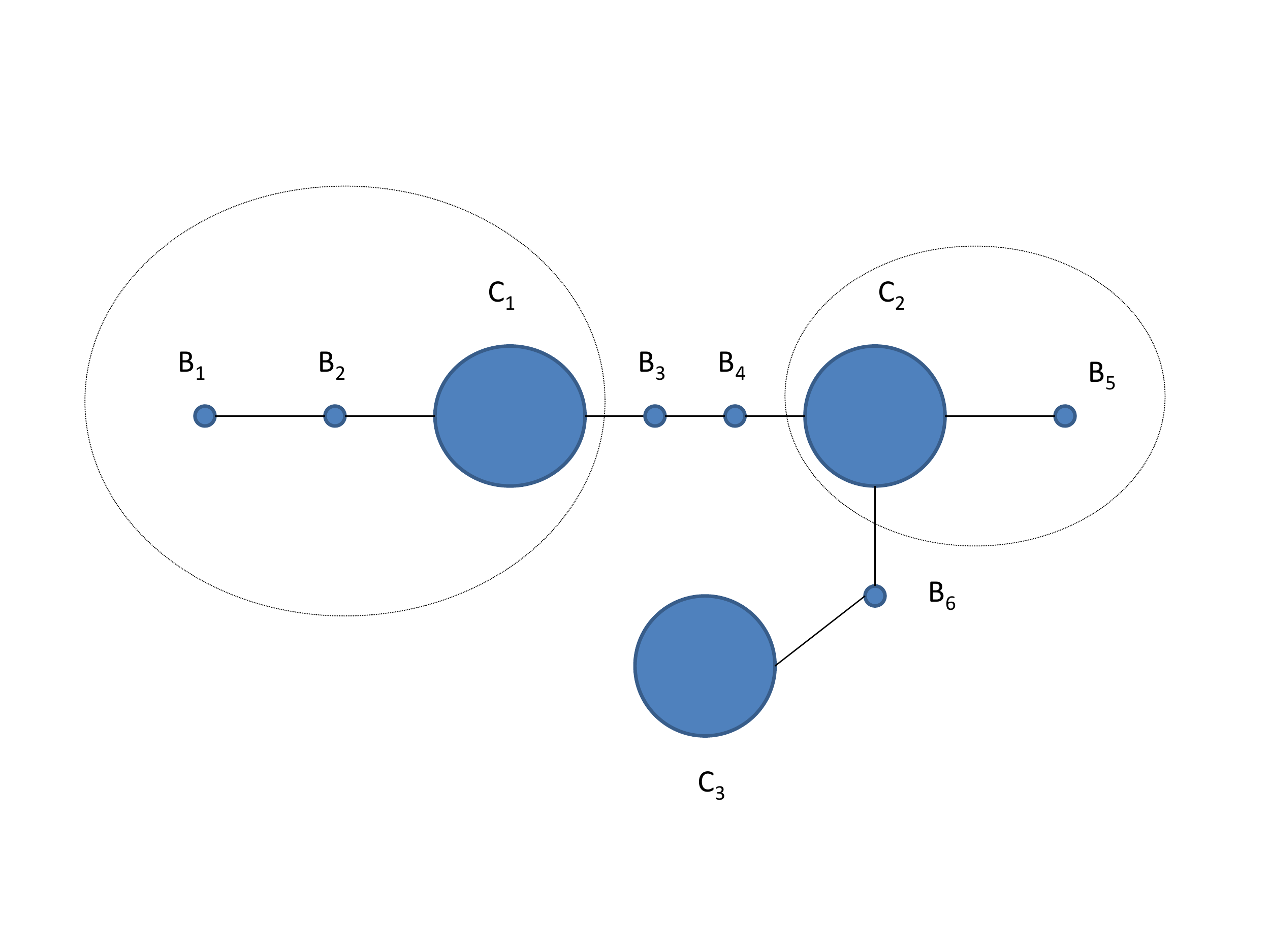}
\caption{Graph $G_\alpha^{t_i}$ with three connected components of big blocks, $C_i$, and six small blocks $B_j$.}
\label{Big_small}
\end{figure}

\begin{ejp}\label{Ejp: Big_small} Suppose $A\in cc(G_\alpha^{t_i})$ is as represented in Figure \ref{Big_small}: $H_\alpha(A)$ has three connected components, $C_i$, $i=1,3$ and $A\backslash H_\alpha(A)$ consists of  six small blocks $B_j$, $j=1,6$. The components $C_j$ are merged by $iii)$. The edges in the figure represent the resulting edges from $G_\alpha^{t_i}$ after identifying the components $C_i$ by $iii)$. 

Now, the component of small blocks formed by $B_1,B_2$ is only adjacent to $C_1$. Therefore, by $iv)$, $C_1\cup B_1\cup B_2$ is contained in some block of $\theta(t_i)$. The same happens with $B_5$ which is a component in $S_\alpha(A)$ which is  only adjacent to the component $C_2$. Thus, $C_2\cup B_5$ is contained in some block of $\theta(t_i)$. However, the component of small blocks given by $B_3,B_4$ is adjacent to two different components of big blocks, $C_1$ and $C_2$. Therefore, $B_3,B_4$ are independent blocks in $\theta_\alpha(t_i)$. The same happens with $B_6$. Thus, $\theta_\alpha(t_i)=\{\{C_1\cup B_1\cup B_2\},\{C_2\cup B_5\},\{C_3\},B_3,B_4,B_6\}$.
\end{ejp}

\begin{obs}\label{Remark: connected} At step $iii)$, if $H_{\alpha}(A)$ is connected, then $B_{i_1}\cup \cdots \cup B_{i_r}$ defines a block of $\theta_\alpha(t_i)$. 
\end{obs}

\begin{obs} \label{Remark: necessary chain 2} Notice that Remark \ref{Remark: necessary chain} still applies. In fact, if two points $x,x'$ belong to the same block  of $\theta^*_{\alpha}(t_i)$  then, necessarily, there exists a $t_i$-chain, $x=x_0,x_1,...,x_n=x'$ joining them so that if $x_j\in B_j\in \theta^*_\alpha(t_{i-1})$, $j=0,...,n$, the corresponding edges $\{B_{j-1},B_j\}$, $j=1,n$, satisfy condition $ii)$.  
\end{obs}

\begin{figure}[ht]
\centering
\includegraphics[scale=0.3]{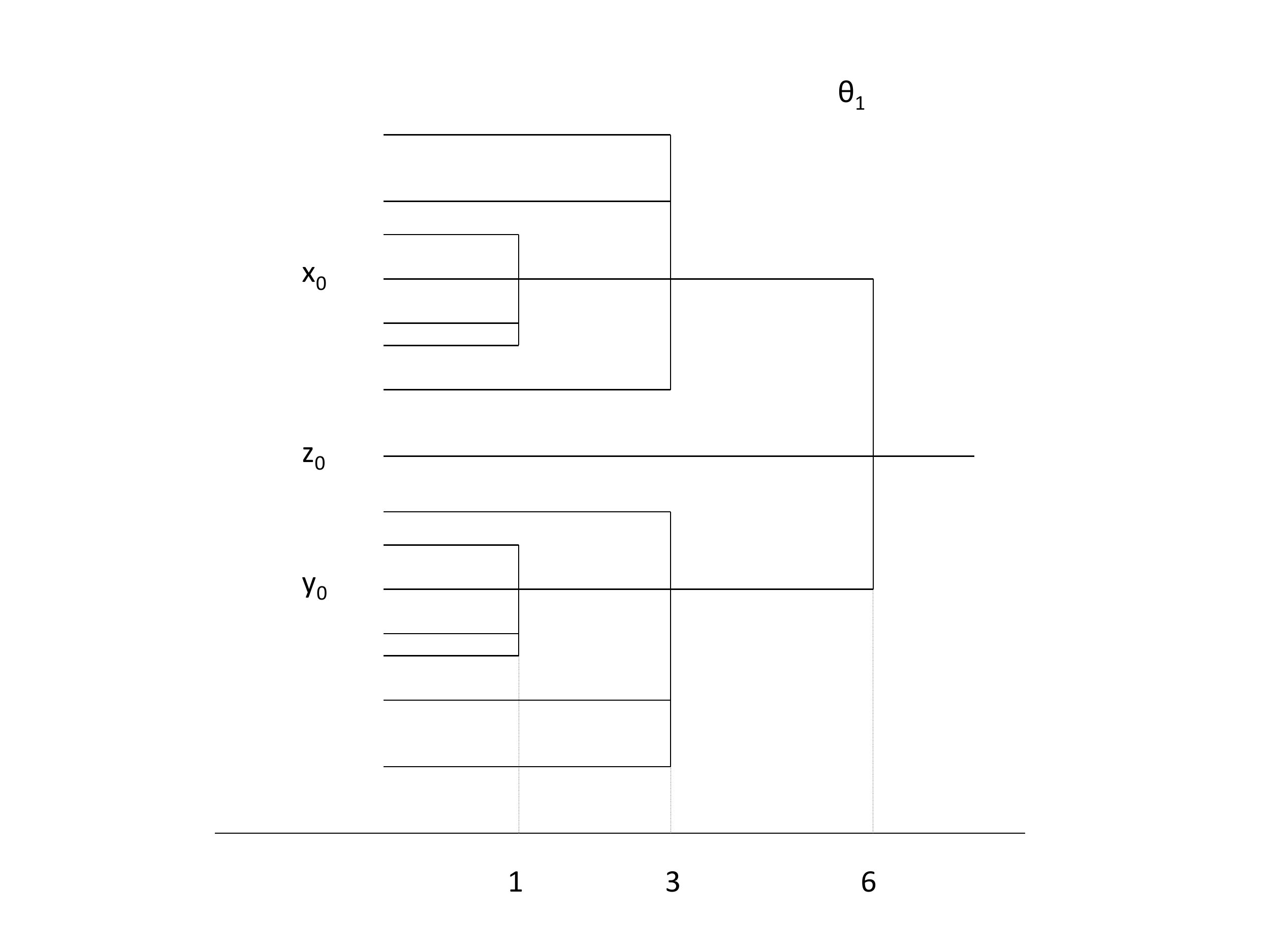}
\caption{Dendrogram produced by $SL^*(\alpha)$ for the graph in Figure \ref{Ejp Chain_2b}}
\label{Chain_dendrogram_2}
\end{figure}

\begin{ejp} \label{Ejp: unchain 1b} Let $(X,d)$ be the graph from Figure \ref{Ejp Chain_2b} and let $\alpha=1$. Then, let us check that applying $SL(1)$ on $(X,d)$ the dendrogram generated is the one from Figure \ref{Chain_dendrogram_2}. Clearly, $\theta_1(t)=\{\{x_0\},...,\{x_{6}\},\{z_0\},\{y_0\},...,\{y_{6}\}\}$ if $t<1$.  If $1\leq t<2$, $\theta_1(t)=\{\{x_1\},\{x_2\},\{x_3\},N_1,\{z_0\},N_2,\{y_1\},\{y_2\},\{y_3\}\}\}$. There are nine $1$-components, seven of them are singletons and two of them, $N_1$, $N_2$, with $\#(N_1)=\#(N_2)=4$. Furthermore, $x_0\in N_1$, $y_0\in N_2$ and $dim\, F_{1}(N_s)=3$ for $s=1,2$.

For $t=2$, conditions $i)$ and $ii)$ induce edges in $G_1^2$ between $N_1$ and $\{z_0\}$ and between $\{z_0\}$ and $N_2$. Then, $G_1^2$ has one component, $A$, which is not a single point: $A=\{N_1, \{z_0\},N_2\}$. $\{z_0\}$ is a single point and $\#N_1=\#N_2=4$. Then, there are two big blocks in $A$, $N_1$ and $N_2$, and one small block, $\{z_0\}$. Since the small block is connected to both big blocks, by condition $iv)$, these blocks are not merged. Thus, for every $1\leq t<3$, $\theta_1(t)=\{\{x_1\},\{x_2\},\{x_3\},N_1,\{z_0\},N_2,\{y_1\},\{y_2\},\{y_3\}\}$.

For $t=3$, there are edges in $G_1^3$ between $N_1$ and $x_i$ for every $1\leq i \leq 3$ and between $N_2$ and $y_i$ for every $1\leq i \leq 3$. Thus, $G_1^3$ is connected. Now, there are two big blocks, $N_1$, $N_2$ and 7 small blocks. By conditions $iii)$ and $iv)$, it is readily seen that $\theta_1(3)=\{B_1,\{z_0\},B_2\}$.

$\theta_1(t)=\{B_1,\{z_0\},B_2\}$ for every $t<6$. The minimal distance $t$ such that $B_1,B_2$ are connected by an edge in $G_1^t$ is $t=6$. $\theta_1(6)=\{X\}$. 
\end{ejp}


\section{Unchaining properties of $SL(\alpha)$}\label{Section: Unchaining properties}

In this section we try to give some theoretical background to the treatment of the chaining effect. Our intention, as it was mentioned above, is to define some concrete element to evaluate the sensitivity of a method to the type of chaining effect we are treating. First, we define the concept of \textit{chained subsets} and subsets \emph{chained by a single edge}. 
   
\begin{definicion} Let $X$ be a finite metric space. We say that two $b$-connected subsets of $X$, $B_1,B_2$, are $(a,b)$-\textbf{chained} if they hold that 
\begin{itemize}
	\item[i)] $\min\{t \ | \  B_1 \mbox{ is $t$-connected }\}=b $,
  \item[ii)] there exist $x_0\in B_1$ and $y_0\in B_2$ such that $d(x_0,y_0)= a \leq b$.
\end{itemize}  

If the parameters $a,\, b$ are not relevant, we say simply that $B_1,B_2$ are \textbf{chained}.
\end{definicion}

Notice that not every case of chained subsets is going to induce an ``undesired'' chaining effect. See Figure \ref{Ejp Chained_2}. The idea is that chained subsets include the cases of undesired chaining effect we are going to treat in this section.

\begin{figure}[ht]
\centering
\includegraphics[scale=0.3]{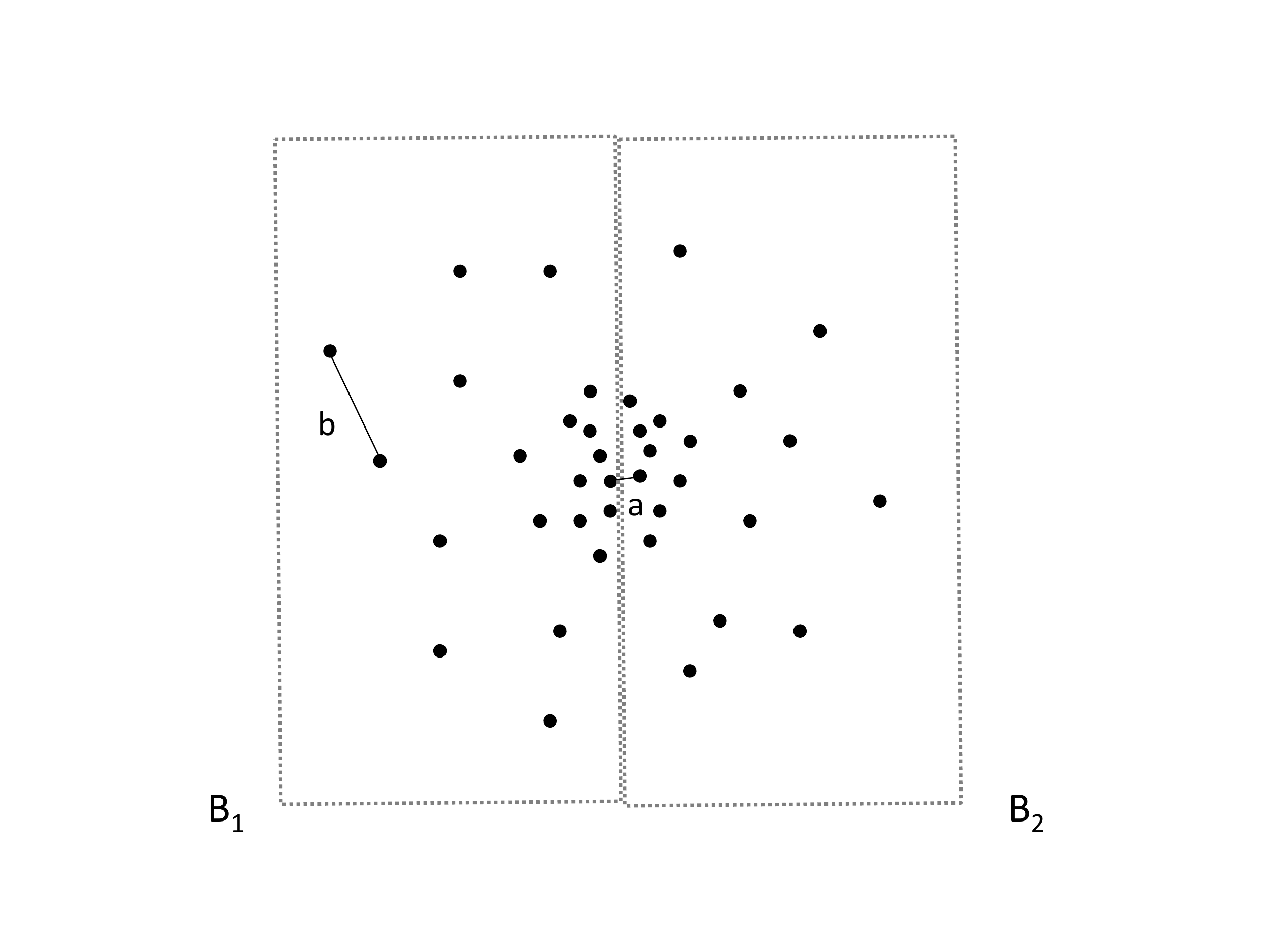}
\caption{$B_1$ and $B_2$ are $(a,b)$-chained subsets although there is no undesired chaining effect if these blocks are merged.}
\label{Ejp Chained_2}
\end{figure}

Inside the cases of chained subsets let us consider some specific type which are the subsets chained by a single edge. We define this as the prototypical case of chained subsets on which the algorithm should be tested to check its unchaining properties. 

\begin{definicion} Let $X$ be a finite metric space. We say that two $b$-connected subsets of $X$, $B_1,B_2$, are $(a,b)$-\textbf{chained by a single edge} if they hold that 
\begin{itemize}
	\item[i)] $\min\{t \ | \  B_1 \mbox{ is $t$-connected }\}=a $,
  \item[ii)] there exist $x_0\in B_1$ and $y_0\in B_2$ such that $d(x_0,y_0)= a \leq b$ 
  \item[iii)] $\forall (x_0,y_0)\neq (x,y) \in B_1\times B_2$, $d(x,y)>b$.
\end{itemize}  

If the parameters $a, \, b$ are not relevant, we say simply that $B_1,B_2$ are \textbf{chained by a single edge}.
\end{definicion}

Suppose that $X=B_1\cup B_2$ and $B_1,B_2$ are $(a,b)$-chained by a single edge. Notice that given the Rips complex $F_{b}(X)$ and the edge $e:=\{x_0,y_0\}\in F_{b}(X)$, then $F_{b}(X)\backslash \{e\}$ has exactly two connected components: $B_1$ and $B_2$. 


\begin{ejp} \label{Ejp: chain_0} The subsets $B_1,B_2$ from Figure \ref{Ejp Chained edge} are $(\varepsilon,t_i)$-chained by a single edge ($\{x_0,y_0\}$). 

$B_1,B_2$ are $t_i$-connected. There is a pair of points $x_0\in B_1$, $y_0\in B_2$ with $d(x_0,y_0)=\varepsilon <t_i$ and for any pair of points $x\in B_1$, $y\in B_2$, if $(x,y)\neq (x_0,y_0)$, then $d(x,y)>\varepsilon$.
\end{ejp}

\begin{ejp} \label{Ejp: chain} Consider the graph represented in Figure \ref{Ejp Chain_1b}. Suppose the edges in $N_1,N_2$ have length $1$ and the rest have length $3$. The distance between the vertices are measured as the length of the minimal path joining them. The whole space is $3$-connected with $d(x_1,x_2)=d(y_1,y_2)=3>1$.

Thus, $B_1$ and $B_2$ are $(3,3)$-chained by a single edge. 
\end{ejp}

\begin{definicion} Let $\mathfrak{T}$ be a HC method and $\mathfrak{T}_\mathcal{D}(X)=\theta$. We say that $\mathfrak{T}$ is \textbf{strongly chaining} if for any set $X$, any pair of chained subsets $B_1,B_2$ of $X$ and any $t>0$, if $B_1$ is contained in some block $B$ of $\theta(t)$, then $y_0\in B$. 
\end{definicion}

\begin{obs} It is immediate to check that $SL$ $HC$ is strongly chaining. Moreover, given any pair of $(a,b)$-chained subsets $B_1,B_2$ such that $X=B_1\cup B_2$  with $a<b$, $\{x_0,y_0\}$ is contained in some block of $\theta_{SL}(a)$ while $B_1$ is not contained in any block of $\theta_{SL}(t)$ for any $t\in [a,b)$. In particular, $\theta(t)$ does not refine $\{B_1,B_2\}$ for any $t\geq a$. 
\end{obs}

\begin{teorema}\label{Th: strongly chaining} Let $\mathfrak{T}$ be a hierarchical clustering method. If for every metric space $X$ and every $x,y,z,t \in X$, $u_{SL}(x,y)\leq u_{SL}(z,t)$ implies that $u(x,y)\leq u(z,t)$, then $\mathfrak{T}$ is strongly chaining. In particular, $SL$ $HC$ is strongly chaining.
\end{teorema}

\begin{proof} First, let us see that $\mathfrak{T}$ is strongly chaining. Consider two $(a,b)$-chained subsets $B_1,B_2$. By hypothesis, there exist  $x_0\in B_1$, $y_0\in B_2$ such that $u_{SL}(x_0,y_0)=a$. Also, there exist $x_1,x_2\in B_1$ with $u_{SL}(x_1,x_2)= b\geq a$. Thus,  $u(x_1,x_2)\geq u(x_0,y_0)$. 

If $B_1$ is contained in some block $B$ of $\theta(t)$, then $t\geq u(x_1,x_2)\geq u(x_0,y_0)$ and $y_0\in B$.
\end{proof}

$AL$ and $CL$ $HC$ are not strongly chaining: 

\begin{ejp}\label{Ejp: not strong chaining} Consider the graph from Figure \ref{Ejp Chain_1b}. Suppose that, in addition, we include edges of length $3$ from $x_1,x_2,x_3$ to every vertex in $N_1$ and from $y_1,y_2,y_3$ to every vertex in $N_2$. Also, suppose that $d(x_0,y_0)=2.5$.

Thus, every pair of points in $N_1$ (resp. $N_2$) are at distance $1$, $d(x_i,x_j)=3$ (resp. $d(y_i,y_j)=3$) for every $i\neq j$, $i,j=0,3$, $d(x_i,x')=3$ for every $x'\in N_1$ and every $1\leq i \leq 3$, $d(y_i,y_j)=3$ for every $i\neq j$, $i,j=0,3$, $d(y_i,y')=3$ for every $y'\in N_2$ and every $1\leq j \leq 3$, $d(x_0,y_0)=2,5$ and $d(x,y)>3$ for every $(x_0,y_0)\neq (x,y)\in B_1\times B_2$. 

Then, $B_1$ and $B_2$ are $(2.5,3)$-chained subsets. However, $\theta_{AL}(1)=\theta_{CL}(1)=\{\{x_1\},\{x_2\},\{x_3\},N_1,N_2,\{y_1\},\{y_2\},\{y_3\}\}$ and $\theta_{AL}(3)=\theta_{CL}(3)=\{B_1,B_2\}$.
\end{ejp}

As we have seen in the examples above, $SL(\alpha)$ is able to detect some partitions of chained blocks. In particular, we have seen that it detects chained blocks when they have dense cores whose distance is greater than the minimal distance between the blocks. This is detected in the Rips complex (for some $t>0$) because there exist high dimensional simplices in both blocks while there is no high dimensional simplex intersecting both. 

Herein, we give sufficient conditions for blocks chained by a single edge to be detected. We also show how the hierarchical clustering is going to recover them. To formalize this we introduce the definition of \textit{weakly unchaining clustering method} and prove that $SL(\alpha)$ is weakly unchaining. 

We also check that other methods as $CL$ or $AL$ $HC$ are not weakly unchaining although, as we mentioned in Example \ref{Ejp: not strong chaining}, they are not strongly chaining either.

\begin{definicion} Let $\mathfrak{T}$ be a HC method and $\mathfrak{T}_\mathcal{D}(X)=\theta$. We say that $\mathfrak{T}$ is \textbf{weakly unchaining for the parameter $\alpha$} if the following implication holds: 

Let $X$ be a finite metric space such that $X=B_1\cup B_2$, with $B_1,B_2$ a pair of subsets $(t_j,t_i)$-chained by a single edge $\{x_0,y_0\}$. Suppose there exist $N_1\in B_1$, $N_2\in B_2$ such that 
	\begin{itemize}
		\item $N_s$ is contained in some block $B^{j-1}_s$ of $\theta(t_{j-1})$, $s=1,2$,
		\item $dim\, F_{t_{j}}(N_s)> \alpha$, $s=1,2$,
		\item $x_0\in N_1$, $y_0\in N_2$,
		\item $\sup_{x,x'\in B_1}\{d(x,x')\}\leq t_i$ and $\sup_{y,y'\in B_2}\{d(y,y')\}\leq t_i$. 
	\end{itemize}

Then, there exists $t>0$ such that $\theta(t)=\{B_1,B_2\}$. 

We say that $\mathfrak{T}$ is \textbf{weakly unchaining} if it is weakly unchaining for some parameter $\alpha$. 
\end{definicion}

\begin{nota} Notice that in the definition above we consider two chained subsets with further conditions. Therefore, if a $HC$ method is strongly chaining, in particular, it is not weakly unchaining.
\end{nota}


\begin{teorema}\label{Th: weakly} Let $X$ be a finite metric space such that $X=B_1\cup B_2$, with $B_1,B_2$ a pair subsets of $(t_j,t_i)$-chained by a single edge $\{x_0,y_0\}$. Suppose there exist $N_1\in B_1$, $N_2\in B_2$ such that 
	\begin{itemize}
		\item $N_s$ is contained in some block $B^{j-1}_s$ of $\theta(t_{j-1})$, $s=1,2$,
		\item $dim\, F_{t_{j}}(N_s)> \alpha$, $s=1,2$,
		\item $x_0\in N_1$, $y_0\in N_2$. 
	\end{itemize}

Then, $\theta_{\alpha}(t_i)$ refines $\{B_1,B_2\}$. If, in addition, $\sup_{x,x'\in B_1}\{d(x,x')\}\leq t_i$ and $\sup_{y,y'\in B_2}\{d(y,y')\}\leq t_i$, then $\theta_{\alpha}(t_i)=\{B_1,B_2\}$. 
\end{teorema}

\begin{proof} Let us recall that, by definition, $t_{j-1}<t_{j}\leq t_{i}$.

For the first part it suffices to check that for every pair $(x,y)\in B_1\times B_2$, $\{x,y\}$ is not contained in any block of $\theta(t_i)$, this is, $u_{\alpha}(x,y)\geq t_i$. 

Let $(x,y)\in B_1\times B_2$. First, notice that for any $t<t_j$, there is no $t$-chain joining $x$ to $y$. Thus, $u_{\alpha}(x,y)\geq t_{j}$. Let us check that $u_{\alpha}(x,y)> t_{j+k}$, $k=0,i-j$. 

For $k=0$, since $x_0\in N_1\subset B^{j-1}_1$ and $y_0\in N_2\subset B^{j-1}_2$, condition $ii)$ implies that there is no edge in $G_\alpha^{t_j}$ between $B^{j-1}_1$ and $B^{j-1}_2$. Since $d(x_1,y_1)>t_{i}$ for every $(x_0,y_0)\neq (x_1,y_1)\in B_1\times B_2$ there is no $t_{i}$-chain joining $x$ to $y$ which does not contain the edge $\{x_0,y_0\}$. In particular, there is no $t_{j}$-chain joining $x$ to $y$ which does not contain the edge $\{x_0,y_0\}$. Therefore, by Remark \ref{Remark: necessary chain}, it follows that $u_{\alpha}(x,y)> t_{j}$. 

The same argument works for every $0<k\leq i-j$. Thus, $u_{\alpha}(x,y)> t_{i}$ and $\theta(t_{i})$ refines $\{B_1,B_2\}$. 

Suppose , in addition, that $\sup_{x,x'\in B_1}\{d(x,x')\}\leq t_i$ and $\sup_{y,y'\in B_2}\{d(y,y')\}\leq t_i$. We already proved that $\theta(t_i)$ refines $\{B_1,B_2\}$. Clearly, since $\sup_{x,x'\in B_1}\{d(x,x')\}\leq t_i$ (respectively, for $B_2$), all the blocks of contained in $B_1$ (resp. $B_2$) are joined by an edge in $G_{\alpha}^{t_i}$. Therefore, $B_1$ (resp. $B_2$) is a block of $\theta(t_i)$.
\end{proof}

\begin{cor}\label{Cor: weakly} $SL(\alpha)$ is weakly unchaining for the parameter $\alpha$. 
\end{cor}

See example \ref{Ejp: unchain 1}.

\begin{cor}\label{Cor: weakly_2} $SL^*(\alpha)$ is weakly unchaining for the parameter $\alpha$. 
\end{cor}


\begin{obs}\label{Ejp: AL-CL} $AL$ and $CL$ $HC$ are not weakly unchaining.  

Consider the graph in Figure \ref{Ejp Chain_1b}. To check that $CL$ $HC$ is not weakly unchaining suppose that we add some edges between $N_1$ and $N_2$ so that $\forall \, (x_0,y_0)\neq (x,y)\in N_1\times N_2$, $d(x,y)=4$. 

Notice that this graph satisfies the conditions in the definition of weakly unchaining. Then, it suffices to check that $\theta_{CL}(t)$ is never $\{B_1,B_2\}$.

It is immediate to check that $\ell^{CL}(N_1,N_2)=4$. Then, it is readily seen that $\theta_{CL}(t)=\{x_1,x_2,x_3,N_1,N_2,y_1,y_2,y_3\}$ for every $1\leq t<4$ and $\theta_{CL}(4)=\{X\}$.  

To check that $AL$ $HC$ is not weakly unchaining suppose that in Figure \ref{Ejp Chain_1b}, we made $d(x_0,y_0)=3-\frac{3}{4}$.
Let us see that $\theta_{AL}(t)$ is never $\{B_1,B_2\}$.

First, notice that this graph satisfies the conditions in the definition of weakly unchaining. 
Also, it is immediate to check that $\ell^{AL}(N_1,N_2)=\frac{3}{4}+3=\ell^{AL}(x_i,N_1)=\ell^{AL}(y_j,N_2)$, $i,j=1,3$. Thus,  it is readily seen that $\theta_{AL}(t)=\{x_1,x_2,x_3,N_1,N_2,y_1,y_2,y_3\}$ for every $1\leq t<3+\frac{3}{4}$ and $\theta_{AL}(3+\frac{3}{4})=\{X\}$.
\end{obs}

DBSCAN is a density-based algorithm for clustering. See \cite{EKSX}. Although this is not a hierarchical clustering method it is worth analyzing the chaining effect and comparing its results with $SL(\alpha)$. 

DBSCAN requires two parameters: some distance $\varepsilon>0$ and a minimal number of points $minPts$. A point is a \textit{core point} of a cluster if there are at least $minPts$ in its $\varepsilon$-neighborhood. Then, the \textit{density-reachable} points from a core point define a cluster. Let us recall here the formal definition.

A \textit{$\varepsilon$-neighbourhood }of a point $p$, $N_\varepsilon(p)=\{q\in X \, | \, d(p,q)\leq \varepsilon\}$. 

A minimum number of points, $minPts$ is defined so that if $N_\varepsilon(p)$ has at least $minPts$, then $p$ is a core point of a cluster.

A point $p$ is \textit{directly density-reachable} from a point $q$ with respect to $\varepsilon$, $minPts$ if 

1) $p\in N_\varepsilon(q)$ and

2) $|N_\varepsilon(q)|\geq minPts$

A point $p$ is \textit{density-reachable} from a point $q$ with respect to $\varepsilon$, $minPts$ if there is a chain of points $q=p_1,...,p_n=p$ such that $p_{i+1}$ is directly density-reachable from $p_i$.

A point is \textit{density-connected }to a point $q$ with respect to $\varepsilon$, $minPts$ if there is a point $o$ such that both, $p$ and $q$ are density-reachable from $o$ with respect to $\varepsilon$, $minPts$.

A cluster is defined to be a set of density connected points which is maximal with respect to density-reachability. The points which do not belong to any clusters are considered noise.

One of the advantages of DBSCAN is that it is capable of detecting isolated points and eliminating them as noise. Also, the chaining effect through a chain of points is reduced.

In fact, the type of unchaining DBSCAN does is more related to the chaining through smaller blocks. It is not so effective to detect the chaining effect produced between two blocks when the minimal distance between them is small. 


In general, suppose $X=B_1\cup B_2$ with $B_1,B_2$ two clusters $(a,b)$-chained by a single edge $\{x_0,y_0\}$. Let us assume that $\varepsilon=a$ and that $x_0,y_0$ are core points. Then, $x_0,y_0$ are density connected and they belong to the same cluster in the output of DBSCAN. Therefore, the clustering $\{B_1,B_2\}$ is not  detected by DBSCAN. See example \ref{Ejp: DBSCAN}. However, if $x_0,y_0$ belong to simplices with dimension at least $\alpha$ in $F_t(B_1)$, $F_t(B_2)$ respectively for some $t<a$, then $SL(\alpha)$  detects this clustering. In fact: $\theta_\alpha(b)=\{B_1,B_2\}$. 

\begin{ejp}\label{Ejp: DBSCAN} Let us analyze the case of example \ref{Ejp: chain}. Let $\varepsilon=3$ and $minPts=4$. Then, notice that $x_0$, $y_0$ are core points. Therefore, since $d(x_0,y_0)=3$, they are density-connected. Thus, DBSCAN does not detect the clustering $\{B_1,B_2\}$.,
\end{ejp}

\section{Unchaining properties of $SL^*(\alpha)$}\label{Section: Unchaining properties 2}

\begin{definicion} Let $X$ be a finite metric space, $B_0,...,B_k$ be $b$-connected subsets of $X$ and $a\leq b$. We say that  $B_0$ and $B_{k}$, are $(a,b)$-\textbf{chained through $\alpha$-smaller blocks} if the following conditions hold 
\begin{itemize}
	\item[i)] $\min\{t \ | \ x\sim_t y \ \forall x,y\in B_0\}=b $, 
  \item[ii)] there exists a $a$-chain  $x_0,...,x_k$ with $x_s\in B_s$ for every $s=0,k$
  \item[iii)] $\forall \, (x,y) \in B_{0}\times B_k$, $d(x,y)>b$.
  \item[iv)] $ \alpha \cdot \#(B_{s})<\min\{\#(B_1),\#(B_k)\}$ for every $1\leq s \leq k-1$.
\end{itemize}  

If the parameters $a,b,\alpha$ are not relevant, we simply say that $B_1,B_2$ are \textbf{chained through smaller blocks}.
\end{definicion}

\begin{definicion} Let $\mathfrak{T}$ be a strongly chaining $HC$ method and $\mathfrak{T}_\mathcal{D}(X)=\theta$. We say that $\mathfrak{T}$ is \textbf{completely chaining} if for any set $X$, any pair of components $B_0,B_k$ of $X$ chained through smaller blocks and any $t>0$, if $B_0$ is contained in some block $B$ of $\theta(t)$, then $\{x_0,...,x_k\}\in B$. 
\end{definicion}

\begin{teorema}\label{Th: complete chaining} Let $\mathfrak{T}$ be a hierarchical clustering method. If for every metric space $X$ and every $x,y,z,t \in X$, $u_{SL}(x,y)\leq u_{SL}(z,t)$ implies that $u(x,y)\leq u(z,t)$, then $\mathfrak{T}$ is completely chaining. In particular, $SL$ $HC$ is completely chaining.
\end{teorema}

\begin{proof} By Theorem \ref{Th: strongly chaining}, we already know that $\mathfrak{T}$ is strongly chaining.

Let $B_0,B_k$ two $b$-connected subsets $(a,b)$-chained through smaller blocks. Let $x_0,...,x_k$ be the corresponding chain. Then, $u_{SL}(x_{r},x_{s})\leq a$ for every $1\leq r,s \leq k$ and there exist $x,x'\in B_0$ such that $u_{SL}(x,x')=b\geq a$. Thus, $u(x,x')\geq u(x_{r},x_{s})$ for every $1\leq r,s \leq k$. 

Now, suppose $t>0$ such that $B_0$ is contained in some block $B$ of $\theta(t)$. Then, $t\geq u(x,x')\geq u(x_{r},x_{s})$ for every $1\leq r,s \leq k$ and $\{x_0,...,x_k\}\in B$.
\end{proof}

\begin{figure}[ht]
\centering
\includegraphics[scale=0.4]{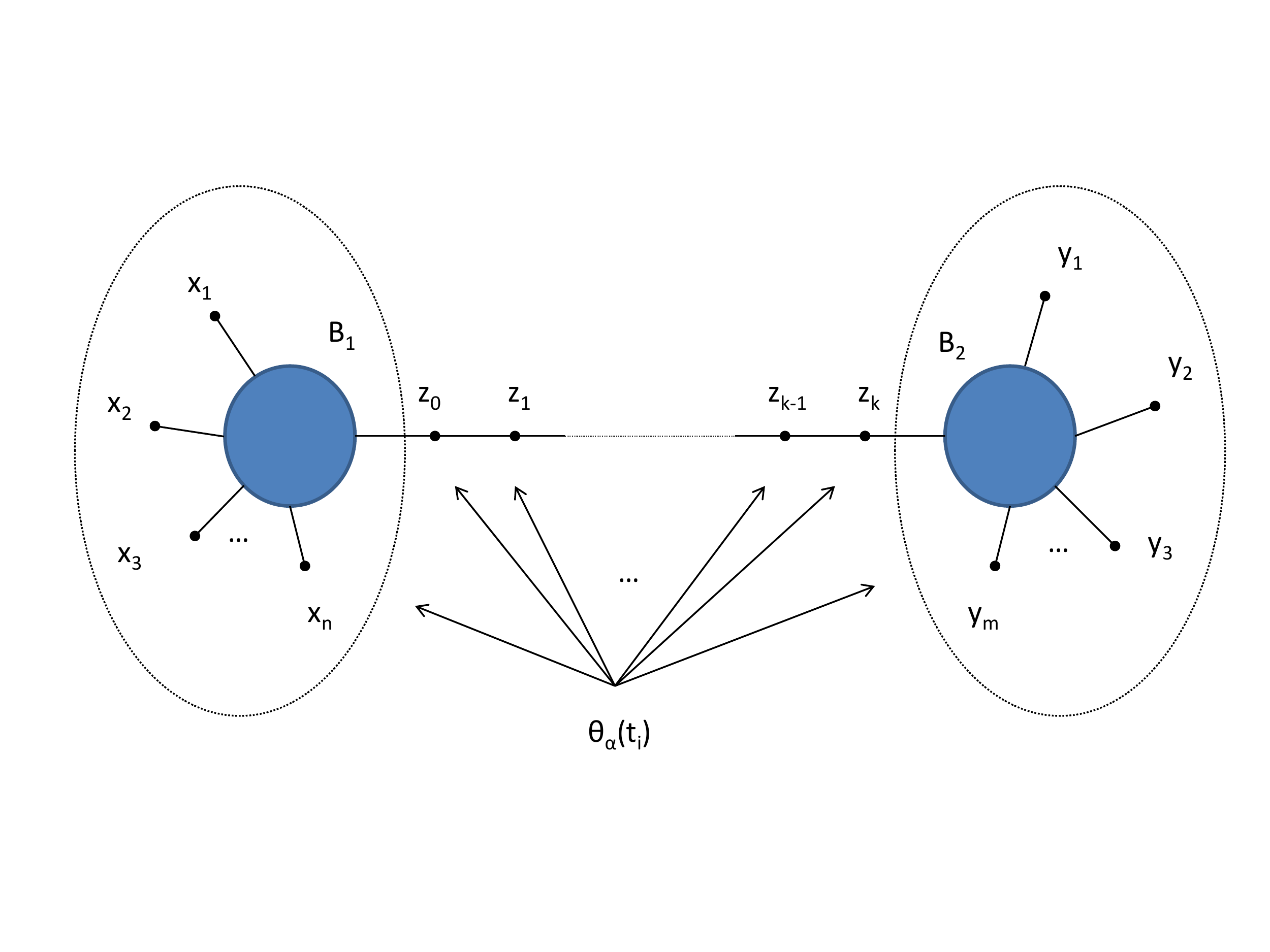}
\caption{If $\theta_{\alpha}(t_{i-1})$ satisfies the conditions from Definition \ref{Def: Moderate} with equalities on conditions $c)$ and $d)$, then $G_\alpha^{t_i}$ is the graph above and $\theta_{\alpha}(t_{i})$ is as indicated.}
\label{Moderate}
\end{figure}

\begin{definicion}\label{Def: Moderate} 
$\mathfrak{T}$ is $\alpha$-\textbf{bridge-unchaining} if it is weakly unchaining for the parameter $\alpha$ and the following implication holds:

Let $X$ be a finite metric space,  $\mathfrak{T}_\mathcal{D}(X)=\theta$ and let $$\theta(t_{i-1})=\{B_1,B_2,\{z_0\},...,\{z_k\},\{x_1\},...,\{x_n\},\{y_1\},...,\{y_m\}\}$$ with $z_j$, $x_r$, $y_s$ single points for every $j,r,s$. Suppose that

	\begin{itemize}
		\item[a)] $d(z_{j-1},z_j)= t_i$ for every $1\leq j\leq k$,
		\item[b)] $d(z_{j_1},z_{j_2})> t_i$ for every $|j_1-j_2|>1$,
		\item[c)] $d(x_r,B_1)\leq t_i$ for every $r$ 
		\item[d)] $d(y_s,B_2)\leq t_i$ for every $s$ 
		\item[e)] $d(z_0,B_1)\leq t_i$ and $d(z_k,B_2)\leq t_i$
		\item[f)] $\min_{1\leq j\leq k-1}\{d(z_j,B_1),d(z_j,B_2)\}>t_i$
		\item[g)] $\min_{r,j,s}\{d(x_r,z_j),d(z_j,y_s),d(x_r,B_2),d(y_s,B_1),d(x_r,y_s),d(B_1,B_2)\}> t_i$ 
		\item[h)] $\alpha<\max\{\#(B_1),\#(B_2)\}$, $\alpha \cdot \#(B_1)> \#(B_2)$ and $\alpha \cdot \#(B_2)> \#(B_1)$.
	\end{itemize}

Then, there exists $t>0$ such that $$\theta(t)=\{\{B_1\cup x_1\cup \cdots \cup x_n\},z_0,...,z_{k},\{B_2\cup y_1\cup \cdots \cup y_n\}\}.$$ 

$\mathfrak{T}$ is \textbf{bridge-unchaining} if it is $\alpha$-bridge-unchaining for some parameter $\alpha$. 
\end{definicion}

\begin{obs} Notice that in the conditions above, if $\min\{t\, | \, x\sim_t y \ \forall \, x,y \in B_1\}=t_i$, then $B_1$ and $B_2$ are $(t_i,t_i)$-chained through the $\alpha$-smaller blocks $z_0,...,z_k$. 
\end{obs}




\begin{figure}[ht]
\centering
\includegraphics[scale=0.4]{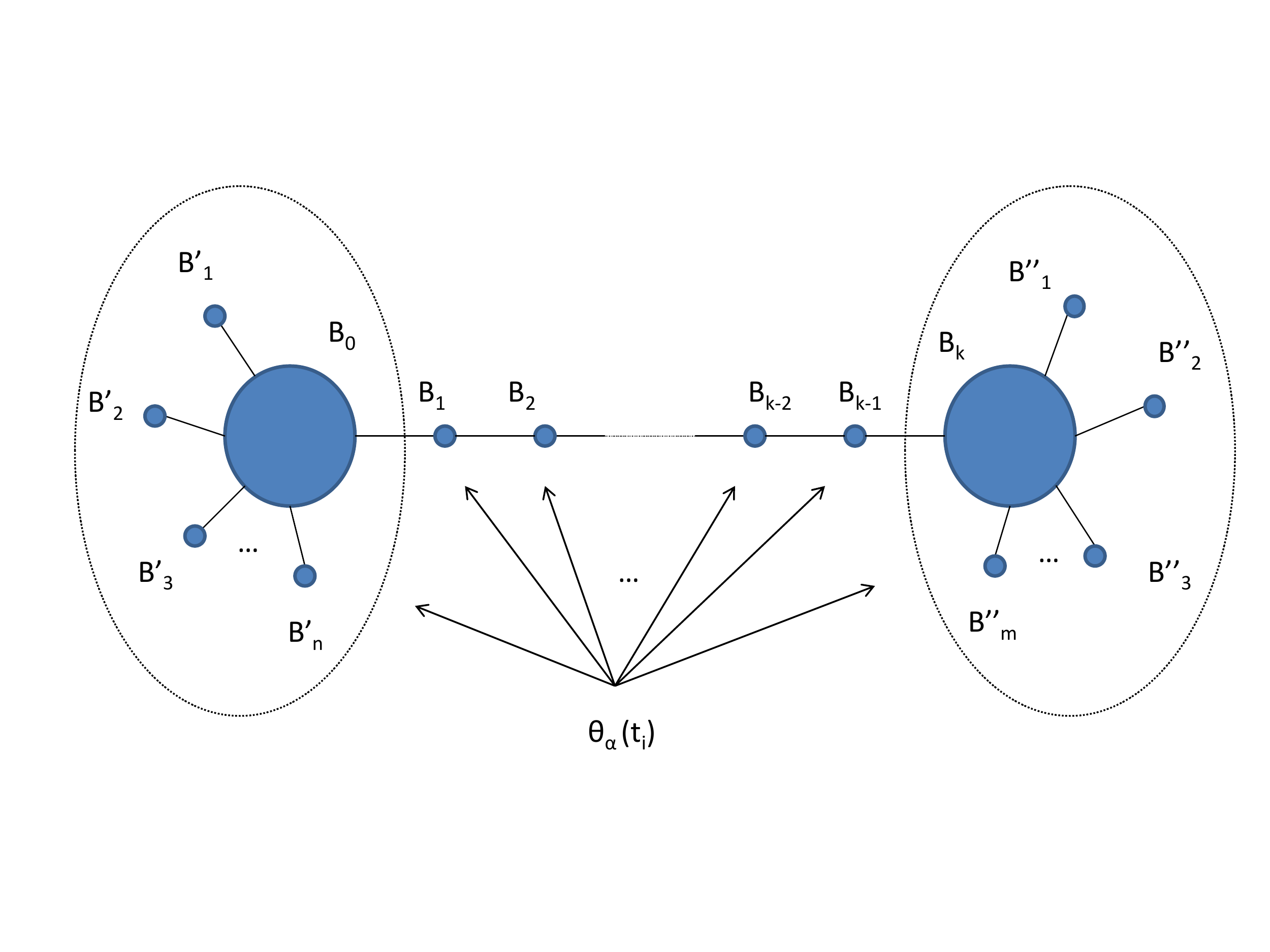}
\caption{If $t_j=t_i$ and $\theta_{\alpha}(t_{i-1})$ satisfies the conditions from Proposition \ref{Prop: moderate_2} with equalities on conditions $c)$ and $d)$, then $G_\alpha^{t_i}$ is the graph above and $\theta_{\alpha}(t_{i})$ is as indicated.}
\label{Moderate_2}
\end{figure}

\begin{teorema}\label{Prop: moderate_2} Let $X$ be a finite metric space and let $$\theta_{\alpha}(t_{j-1})=\{B_0,B_1,...,B_{k-1},B_k,B'_1,...,B'_n,B''_1,...,B''_m\}$$ with $t_j\leq t_i <2t_j$. Suppose that
	\begin{itemize}
		\item[a)] $d(B_{\ell-1},B_\ell)= t_j$ for every $1\leq \ell\leq k$,
		\item[b)] $d(B_{\ell_1},B_{\ell_2})> t_i$ for every $|\ell_1-\ell_2|>1$,
		\item[c)] $d(B'_r,B_0)\leq t_i$ for every $r$ 
		\item[d)] $d(B''_s,B_k)\leq t_i$ for every $s$ 
		\item[e)] $d(B'_r,B_\ell)> t_i$ for every $r$ and  every $1\leq \ell \leq k$
		\item[f)] $d(B_\ell,B''_s)> t_i$ for every $s$ and  every $0\leq \ell \leq k-1$
		\item[g)] $\alpha \max_{1\leq \ell\leq k-1}\{\#(B_\ell)\}<\max\{\#(B_0),\#(B_k)\}$, $\alpha \cdot \#(B_0)> \#(B_k)$ and $\alpha \cdot \#(B_k)> \#(B_0)$.
		\item[h)] $\alpha >dim(F_{t_i}(B_\ell))$ for every $1\leq \ell \leq k-1$, $\alpha >dim(F_{t_i}(B'_r))$, $\alpha >dim(F_{t_i}(B''_s))$  for every $r,s$.
 	\end{itemize}

Then, $$\theta_{\alpha}(t_i)=\{\{B_0\cup B'_1\cup \cdots \cup B'_n\},B_1,...,B_{k-1},\{B_k\cup B''_1\cup \cdots \cup B''_m\}\}.$$ 
\end{teorema}

\begin{proof} Let $$\theta_{\alpha}(t_{j-1})=\{B_0,B_1,...,B_{k-1},B_k,B'_1,...,B'_n,B''_1,...,B''_m\}$$ satisfying the conditions above. For $t=t_j$ let us apply conditions $i)$ and $ii)$ of $SL(\alpha)$. Since, $\alpha >dim(F_{t_i}(B_\ell))>dim(F_{t_j}(B_\ell))$ for every $1\leq \ell \leq k-1$, we obtain edges $\{B_{\ell-1},B_\ell\}$ for every $1\leq \ell\leq k$. Since $\alpha >dim(F_{t_i}(B'_r))$, $\alpha >dim(F_{t_i}(B''_s))$  for every $r,s$, we also obtain edges $\{B'_r,B_0\}$ and $\{B''_s,B_k\}$ for every $r,s$ such that the distance is less or equal than $t_j$. Thus, the blocks $B_\ell$, $0\leq \ell \leq k$ are in the same connected component of $G_\alpha^{t_j}$. Since $\alpha \max_{1\leq \ell\leq k-1}\{\#(B_\ell)\}<\max\{\#(B_0),\#(B_k)\}$, by $iv)$, $B_\ell$ is an independent block in $\theta_{\alpha}(t_j)$ for every $1\leq \ell \leq k-1$. Also, by $iv)$, the blocks $B'_s$ joined by an edge to $B_0$ (resp. the blocks $B''_r$ joined by an edge to $B_k$) are merged with $B_0$ (resp. $B_k$). 

The same argument holds for every $t_j<t\leq t_i$. Thus,
$$\theta_{\alpha}(t_i)=\{\{B_0\cup B'_1\cup \cdots \cup B'_n\},B_1,...,B_{k-1},\{B_k\cup B''_1\cup \cdots \cup B''_m\}\}.$$ 
%
%
\end{proof}

Renaming the corresponding blocks, it is immediate to obtain the following:

\begin{cor}\label{Cor: unchaining} $SL^*(\alpha)$ is $\alpha$-bridge unchaining.
\end{cor}

See example \ref{Ejp: unchain 1b}.

\begin{obs}\label{AL CL isolated} $CL$ and $AL$ are not bridge-unchaining. As we mentioned in the introduction, $CL$ and $AL$ show a tendency to merge isolated points before joining them to a pre-existing cluster. This tendency can be seen aplying the algorithms in the following cases.

Let $\theta_{CL}(t_{i-1})=\{B_1,x_1,...,x_n,z_0,..., z_{k},B_2, y_1,..., y_n\}$ and suppose the conditions from Definition \ref{Def: Moderate} hold with equalities on conditions $c)$ and $d)$. See Figure \ref{Moderate}. Assuming that $\ell^{CL}(z_0,B_1)>t_i$ and $\ell^{CL}(z_k,B_2)>t_i$, then in $\theta_{CL}(t_i)$, $\{z_0\cup \cdots \cup z_{k}\}$ is a cluster.

Similarly, let $\theta_{AL}(t_{i-1})=\{B_1,x_1,...,x_n,z_0,..., z_{k},B_2, y_1,..., y_n\}$ and suppose the conditions from Definition \ref{Def: Moderate} hold with equalities on conditions $c)$ and $d)$. Assuming that $\ell^{AL}(z_0,B_1)>t_i$ and $\ell^{AL}(z_k,B_2)>t_i$, then in $\theta_{AL}(t_i)$, $\{z_0\cup \cdots \cup z_{k}\}$ is a cluster. 

In either case, the point $z_0$ (resp. $z_k$) would appear to be far away from $B_1$ (resp. $B_2$) while $z_0$, $z_k$ appear to be very close to each other.
\end{obs}

This type of chaining through points or through smaller blocks may be also avoided by DBSCAN. In the conditions from Definition \ref{Def: Moderate}, if $\varepsilon<d(B_1,B_2)$ and the points $z_0,...,z_k$ are not core points (with $k>0$), DBSCAN would not merge $B_1$ and $B_2$ either. However, it would not necessarily return the clustering $\{\{B_1\cup x_1\cup \cdots \cup x_n\},z_0,...,z_{k},\{B_2\cup y_1\cup \cdots \cup y_n\}\}$. 

The result obviously depends on the density distribution of the points in $B_1$ and $B_2$ and the  parameters $\varepsilon$, $minPts$ involved. Let us assume that every point in $B_1$, $B_2$ is a core point and $\varepsilon \geq t_i$. In this case, the output of DBSCAN will be two clusters, $B,B'$ with $\{B_1\cup x_1\cup \cdots \cup x_n\cup z_0\}\subset B$, $\{B_2\cup y_1\cup \cdots \cup y_n\cup z_k\}\subset B'$ and some single points (noise) from the sequence $z_1,...,z_{k-1}$



\section{Discussion}\label{Section: Conclusions}

Herein, we treat a particular type of chaining effect which is characteristic from single linkage. This effect is reduced if the algorithm shows some sensitivity to the density distribution of the data set. This is why average linkage or complete linkage are usually preferred by practitioners.  Our aim was to define an algorithm  such that it encodes information about the density distribution with a very simple input. $SL(\alpha)$ is able to detect clusters affected by this kind of chaining effect.

We also provide some theoretical background to the study of the chaining effect. 
Thus, a hierarchical clustering method is \textit{strongly chaining} if every pair of chained clusters is automatically merged in one cluster. This is the case of single linkage. On the contrary, a hierarchical clustering method is \textit{weakly unchaining} if at least it detects some type of chained clusters when they have dense nuclei distant apart.  We prove that  $SL(\alpha)$ is weakly unchaining while complete linkage and average linkage are not. Also, compared with DBSCAN, our method seems to have a more natural and powerful treatment of this problem.  

One weakness of $SL(\alpha)$ is that it fails to detect when two blocks are chained by a single point or a small block. $SL^*(\alpha)$ deals with that weakness. We prove that $SL^*(\alpha)$ is  $\alpha$-\textit{bridge-unchaining} showing that $SL^*(\alpha)$ is capable of detecting this kind of chaining.

We focused on the theoretical problem of chaining so 
we did not consider the computational problem involved. It would be interesting to study how efficient is $SL(\alpha)$ compared with other methods.

One of the main advantages of $SL$ is that it is stable in the Gromov-Hausdorff sense which is a really interesting property for a clustering algorithm. If the algorithm is too sensitive to small perturbation of the data the output may  be easily meaningless. Unfortunately, our method does not share with $SL$ the good stability properties. Modifying the algorithm to deal with the chaining effect we lost that advantage. The problem of stability of linkage-based clustering methods  and the difficulties to define algorithms, other than $SL$, stable in the Gromov-Hausdorff sense is studied in \cite{M}.

\section{Acknowledgments}

The author would like to express his gratitude to Bruce Hughes for pointing out such a nice paper to read.

\end{document}